\documentclass[sn-mathphys-num,iicol]{sn-jnl}

\usepackage{graphicx}%
\usepackage{multirow}%
\usepackage{amsmath,amssymb,amsfonts}%
\usepackage{amsthm}%
\usepackage{mathrsfs}%
\usepackage[title]{appendix}%
\usepackage{color}%
\usepackage{bm}

\newtheorem{thm}{Proposition}
\newtheorem{lem}{Lemma}
\newtheorem{mthdef}{Definition}

\DeclareMathOperator*{\SE3}{\mathnormal{SE}(3)}

\DeclareMathOperator*{\Exp}{\mathrm{Exp}}

\DeclareMathOperator*{\tr}{^{\top}}
\DeclareMathOperator*{\ohat}{\theta^{\land}}

\begin{document}

\title[Eigen-Factors a Bilevel Optimization for Plane SLAM of 3D Point Clouds]{Eigen-Factors a Bilevel Optimization for Plane SLAM of 3D Point Clouds}


\author*[1]{\fnm{Gonzalo} \sur{Ferrer}}\email{g.ferrer@skoltech.ru}

\author[2]{\fnm{Dmitrii} \sur{Iarosh}}\email{dimajarosh@gmail.com}

\author[1]{\fnm{Anastasiia} \sur{Kornilova}}\email{anastasiia.kornilova@skoltech.ru}

\affil[1]{\orgdiv{Center for AI technology}, \orgname{Skolkovo Institute of Science and Technology
(Skoltech)}, \orgaddress{\street{Bolshoy Boulevard, 30, p.1}, \city{Moscow}, \postcode{121205}, \country{Russia}}}

\affil[2]{\orgdiv{Software Engineering Department}, \orgname{St Petersburg State University}, \orgaddress{\street{7-9 Universitetskaya Embankment}, \city{St Petersburg}, \postcode{199034},  \country{Russia}}}


\abstract{
Modern depth sensors can generate a huge number of 3D points in few seconds to be later processed by Localization and Mapping algorithms.
Ideally, these algorithms should handle efficiently large sizes of Point Clouds (PC) under the assumption that using more points implies more information available.
The Eigen Factors (EF) is a new algorithm that solves PC SLAM by using planes as the main geometric primitive. To do so, EF exhaustively calculates the error of all points at complexity {\em O(1)}, thanks to the {\em Summation matrix S} of homogeneous points.

The solution of EF is a bilevel optimization where the {\em lower-level} problem estimates  the plane variables in closed-form, and the  {\em upper-level} non-linear problem uses second order optimization to estimate sensor poses (trajectory).
We provide a direct analytical solution for the gradient and Hessian based on the homogeneous point-plane constraint. In addition, two variants of the EF are proposed: one pure analytical derivation and a second one approximating the problem to an alternating optimization showing better convergence properties.


We evaluate the optimization processes (back-end) of EF and other state-of-the-art plane SLAM algorithms in a synthetic environment,  and extended to ICL dataset (RGBD) and LiDAR KITTI datasets. EF demonstrates superior robustness and accuracy of the estimated trajectory and improved map metrics.
Code is publicly available at \texttt{https://github.com/prime-slam/EF-plane-SLAM} with python bindings and pip package.
}




\maketitle

\section{Introduction}\label{sec_intro}

The 
advent of modern sensing capabilities is offering the possibility to perceive 3D depth, i.e. Point Clouds (PC), at a high streaming rate from diverse sensors, such as 3D LiDARs, Time of Flight cameras, RGB-D, etc.
The more points are used, the more information is gathered from the environment.
However, is it possible to evaluate all observed points and improve accuracy without a considerable increase in complexity?
This is the main question this paper tries to address.

Aligning a sequence of observed PCs can be formulated as a Simultaneous Localization and Mapping (SLAM) \cite{dellaert2006} problem, where full trajectory and map geometry are purely estimated from the sensor's observed 3D points.
Similarly, this problem definition fits perfectly the Bundle Adjustment (BA) \cite{triggs1999bundle} problem. 
In this paper, we are interested in the design and study of a new observation function that allows to formulate the SLAM (or BA) problem by directly considering the sensor's observed points in an efficient manner.

An important hindrance is that the sensor's points from PCs are sampled from the scene geometry, consequently it is highly unlikely to observe the same exact point on later observations.
Therefore, one needs to define some representation of this inherent geometry that is capable of providing the error-point in an efficient manner.
We choose planar surfaces as landmarks due to their regularity and simple constraint equation. These geometrical surfaces are commonly present in indoors and urban areas.
Plane SLAM \cite{weingarten2006,trevor2012} approaches use planes as landmarks, but the observations are pre-processed into planar features, hence,  3D points are not directly used with their consequent performance deterioration. Common current challenges in  plane-SLAM methods include planar feature extraction or handling large-scale data. This paper is focused on the latter.


We propose a plane SLAM technique for back-end optimization, the Eigen-Factors (EF). While the front-end of SLAM deals with processing the raw data by sensors, the back-end focuses on the optimization processes, which is exactly the aim of this paper for plane SLAM. EF is tailored for 3D Point Clouds (PC) from RGBD and LiDAR, while being capable of processing a large amount of points with constant complexity.
This paper is an improved version of the former EF \cite{ferrer2019}, overcoming its main limitation (first order optimization).
The EF uses planar geometric constraints for the alignment of a trajectory described by a moving sensor. To do so, EF aggregates all observed points into the {\em Summation matrix} $S$, a summation of outer product of homogeneous points and directly calculates the exact point error at $O(1)$ complexity.

In addition, EF optimizes the point to plane error w.r.t. the trajectory. The plane variables are calculated in closed form in a {\em lower-level} problem and the trajectory estimation corresponds to the {\em upper-level} problem in a bilevel optimization program \cite{colson2007overview}.

The gradient and the Hessian are directly calculated by reformulating the problem in terms of manifold optimization on $\SE3$ and retractions. Unfortunately, EF does not allow for the Gauss-Newton optimization approximation, since there are not residuals but eigenvalues, so instead we use second order optimization.
In order to differentiate over $\SE3$, we have proposed a methodology which greatly simplifies the calculation of derivatives of any order and it could be used in other problems where an analytical expression is complex to derive, as in our case through eigenvalues or matrix elements.

We provide two solutions for the EF trajectory optimization: One is an analytical and exhaustive derivation of the Hessian, which is a {\em dense} matrix and the second is an approximation by considering the problem to be {\em alternating}, i.e. the variable set on each of the loops are independently estimated.
The ideal application of EF is map refinement of the 3D PC, therefore it requires a relatively good initial estimate of the trajectory.
Both EF variants provide accurate results, and it will be discussed later the advantages and disadvantages of each method.

The contributions of the present paper are the following:
\begin{itemize}
    \item Matrix summation $S$ to calculate all point to plane errors at $O(1)$ complexity;
    \item Bilevel optimization where first the plane parameters are calculated in closed form and then {\em upper-level} variables  (only trajectory) are estimated with non-linear second order optimization;
    \item Analytical calculation of the gradient and the Hessian for the EF (Dense) and an alternating optimization approximation;
    \item Methodology to calculate derivatives of any function involving the manifold of rigid body transformations $\SE3$;
\end{itemize}

The evaluation compares EF with other state-of-the-art plane SLAM back-end algorithms in synthetic, RGBD and LiDAR environments, since we believe these are the most useful settings for plane SLAM.
In order to clearly understand all the details of each method, we have decided to compare only the back-end and provide to all methods the same exact input. The front-end has an impact to the final performance of the methods so we have decided to isolate this effect.


\section{Related Work}\label{sec_rw}

The initial motivation for EF is to align PC minimizing the point error.
The Iterative Closest Point (ICP) algorithm \cite{chen1991,besl1992, zhang1994,rusinkiewicz2001} fits a pair of PCs by iteratively finding correspondences and solving the optimization problem \cite{horn1987,arun1987,umeyama1991}.
Existing solutions include matching point-to-point, point-to-line \cite{censi2008}, point-to-plane, and other variants emulating planes \cite{segal2009,serafin2015} or assuming a general geometry. These techniques require to process all points every iteration and can achieve high accurate results.

One can bring this idea to a sequence of PC or Multi-view registration. An early work \cite{bergevin1996towards} proposed to match all points in every view, reducing the overall point error, but requiring large  memory resources and intense computations.
Improvements on Multi-view registration targeted efficiency, for instance considering pairs of PC as constraints \cite{pulli1999multiview}.

Although computing the overall point error is a natural choice for mapping and SLAM, researchers have been trying to reduce the {\em a priory} high computational requirements of these techniques.
On this regard, processing plane landmarks for each PC observation could be seen as a proxy for computing the point error but at a reduced computational cost.
Therefore, plane Landmark-based SLAM optimizes trajectory and a  map of plane-landmarks \cite{weingarten2006,trevor2012}.
There have been proposed multitude of variants, some of them focusing on the plane representation \cite{kaess2015,geneva2018}, 
adding planar constraints \cite{zhang2019}  
or a mixture of points and planes \cite{taguchi2013,zhou2021lidar}.
The key problem of plane landmarks is that after extracting planar features from noisy points, there is an inevitable loss of information compared to considering jointly all points. EF on the other hand, keeps all points' information.

Modern mapping approaches, in order to ensure real-time (or close) conditions, aggregate observations to the map after an accurate scan-to-map, effectively marginalizing poses.
Examples of these approaches include  LOAM \cite{zhang2014,zhang2017} which combines PC alignment with mapping or LIO-mapping \cite{ye2019}.
Surfel-based approaches \cite{whelan2016,behley2018} follow the same principle dividing the scene in planar disks, obtaining impressive results.
When considering loop closure, pure mapping approaches present limitations and require an extra effort to update their mapping solutions properly.
SLAM approaches do not suffer this limitation, since in general both the trajectory and the map are estimated simultaneously, so they handle more naturally loop closure situations at the cost of higher computational resources.

When considering overall point error, efficiency is of the essence.
Our previous paper on EF \cite{ferrer2019} introduces for the first time the $S$ matrix as a summation of outer products the homogeneous points from each sensor pose.
This allows to optimize the point to plane error with complexity $O(1)$, which is the key requirement for our approach to work efficiently.
Later, $\pi$-LSAM \cite{zhou2021pi} makes use of the same matrix $S$ to calculate the overall point to plane error to solve the Planar Bundle adjustment estimating jointly poses and plane-landmarks.
This approach is shown to perform worse  (Sec.~\ref{sec_rbgd}) under the presence of noise than point-error methods like EF.
In a following work, the authors refer to this homogeneous summation matrix $S$ as the {\em Integrated Cost Matrix} (ICM) \cite{zhou2021lidar}.

The BAREG \cite{huang2021} algorithm uses the functional proposed in \cite{ferrer2019} and updates incrementally a covariance matrix $3\times 3$, without explicit re-evaluation of points, maintaining the same advantages as the $S$ matrix.
BALM2 \cite{liu2022efficient} uses the same concept of the $S$ matrix, which they refer as {\em Cluster Points} and formalizes some properties of the $S$ matrix such as superposition and transformation.

The idea of using the eigendecomposition from EF as an objective function is adopted by
BALM \cite{liu2021balm}, whose derivation is from a $3\times 3$ covariance instead of the $4\times 4$ homogeneous $S$ matrix as in EF. The authors improve the efficiency \cite{liu2022efficient} by calculating analytical derivatives of the gradient and Hessian.
We empirically show how a tighter estimation of the Hessian, as proposed by BALM, does not necessarily yield better results especially when evaluating under perturbations, whereas the EF analytical approximation to the Hessian performed more accurately.

In BAREG\cite{huang2021}, the authors use the eigendecomposition to demonstrate the equivalence of the point-to-plane error to the plane-to-plane error, with proper weights. This equivalence assumes that the plane parameters are optimal which might be a too restrictive condition in practice. However, this reduction greatly reduces the complexity and efficiency of the optimization step. It also requires the plane estimation at each observation, which has consequences when evaluating under strong point noise conditions.


In this work, we are interested also in assessing the quality of the map by using a mapping metric. 
The ground truth on some datasets will not be accurate enough for the mapping purposes, so these assessment tools are necessary.
Whereas there are many solutions for precise trajectory estimation from sequences of point clouds, the topic of measuring map quality has not received the same attention. 
The 3D graphics community proposes a set of full reference metrics for solving tasks of point cloud compression and denoising, including p2point metric, p2plane metric~\cite{tian2017}, angular similarity~\cite{alexiou2018}, projection based methods~\cite{torlig2018} or SSIM \cite{meynet2019,alexiou2020}.
The majority of metrics based on point cloud requires reference to ground truth, whereas in the evaluation scenario we propose (KITTI), 3D reconstruction ground truth will not be available. Therefore no-reference evaluation is used in this work. In particular,  Mean Map Entropy (MME) \cite{droeschel2014local} and Mean Plane Variance (MPV) \cite{razlaw2015evaluation} evaluate noise in the aggregated point cloud and Mutually Orthogonal Metric (MOM) \cite{kornilova2021} correlates reconstruction noise with perturbation in the 3D poses.


\section{Eigen-Factors: Point to Plane Error}\label{sec_bg_plane}

A plane $\pi$ in the 3D space is defined by a  normal unit vector $\eta \in \mathbb{S}^2$ and the plane distance to the origin $d \in \mathbb{R}$:

\begin{equation}
\pi = \begin{bmatrix}\eta\\ d\end{bmatrix} \in \mathbb{P}^3, \qquad \text{where   } ||\eta|| =1.
\label{eq_plane_def}
\end{equation}

A point  $p = [x,y,z]^{\top}$ lies on the  plane $\pi$ if an only if:
\begin{equation}
\pi^{\top} \begin{bmatrix}p\\ 1\end{bmatrix} = \pi^{\top} \tilde{p} = 0.
\label{eq_plane_constraint}
\end{equation}

With this constraint and given a set of noisy observed points from the same planar surface, one can solve the plane estimation  mainly by two approaches.
The {\bf centered method} minimizes the overall point to plane error
\begin{equation}
\min_{\pi} \sum_{n=1}^N|| \eta^{\top} p_n + d ||^2.
\end{equation}

The solution comes in two steps:
\begin{gather}
\min_{\eta} \left\{ \sum_{n=1}^N \eta^{\top}  \underbrace{(p_n - E\{p\}) (p_n - E\{p\})}_{\Sigma_p} \eta  \right\} \\
d = - E\{\eta^{\top} p\},
\label{eq_center_estimate}
\end{gather}
first the eigendecomposition of the $3\times 3$ matrix covariance $\Sigma_p$ solves the value of $\eta$ and then $d$ is obtained.

The {\bf homogeneous method} calculates the point to plane error of
a plane $\pi$ without centering the data points:
\begin{equation}
\min_{\pi} \sum_{n=1}^N|| \pi^{\top} \tilde{p}_n||^2 = \pi^\top \tilde{P} \tilde{P}^\top \pi
\label{eq_homogeneous_estimate}
\end{equation}
where the matrix $\tilde{P}$ is the stacked vector of $N$ homogeneous points. This arrangement of terms in the homogeneous plane estimation method allows us to define the {\em Summation} matrix:
\begin{equation}
S := \tilde{P} \tilde{P}^\top = \sum_{n=1}^N \tilde{p}_n \tilde{p}_n^{\top},
\label{eq_s_matrix}
\end{equation}
which is the sum of the outer product of all points in the plane.
The solution of (\ref{eq_homogeneous_estimate}) is calculated using the eigendecomposition of the $4\times 4$ matrix $S$.
In \cite{klasing2009} there is a comparison of the two methods yielding similar results. However, we found that this statement hold for some conditions, namely when near the origin, and when calculating the derivatives, the centered approach is numerically more stable (see Sec.~\ref{sec_exp_numerical}).

Our ultimate objective is the optimization of the trajectory and the geometrical features in the map. The disadvantage of the {\em centered method} is that every time a new sample is added or modified, all calculations should be carried out again.
It is more natural to develop the {\em homogeneous method} to account for poses in a highly efficient solution. We will see how it is possible to derive an homogenous-based approach while maintaining the numerical advantages of the centered method.

Therefore, the Eigen-Factors is a generalization of the homogeneous plane estimation observed over a sequence of reference frames or poses  $T_t \in \SE3$ at different instants of time $t \in \{ 1,\ldots,H\}$.
It can also be understood as a trajectory $\bm{T} = \{T_1,\ldots ,T_H\}$.
Each of these poses transforms points to a global reference frame $g$, such that ${}^gT_t$, where we will omit the global frame reference for simplicity in the following sections.
The underlying idea is that the same plane $\pi$ is observed from different poses $T_t$ but the planar constraint must hold for all views.
Accordingly, we can reformulate the plane estimation problem for the plane $\pi$ as
\begin{equation}
\min_{\pi} \sum_{t=1}^H \sum_{n=1}^{N_t} || \pi^{\top} T_t \tilde{p}_{t,n}||^2 
 = \min_{\pi} \pi^{\top} \big(\sum_{t=1}^H T_t S_t T_t^{\top}\big) \pi,
\label{eq_ef_plane_estimation}
\end{equation}
where each of the points $\tilde{p}_{t,n}$ denotes the reference frame $t$ it was observed and the matrix $S_t$ includes all these $N_t$ points as in (\ref{eq_s_matrix}).
By re-arranging terms one can see as the plane $\pi$ is out of the summation, so the plane solution boils down to the eigendecomposition of this new matrix.
Therefore, we define the matrix $Q$ as the summation
\begin{equation}
Q(\bm{T}) :=  \sum_{t=1}^H T_t S_t T_t^{\top} = \sum_{t=1}^H Q_t(T_t),
\label{eq_q}
\end{equation}
where each component $Q_t$ depends on the matrix $S_t$, {\bf calculated  only once}, and the current 3D pose estimation $T_t$.
The matrix $S_t$ expresses points in the local coordinate system at time $t$ and is constant, while the matrix $Q$ is expressed in the global reference frame and is recalculated for  each new updated trajectory.

If one inspects the elements of the matrix $Q$, the following terms appear:
\begin{equation}
Q = \begin{bmatrix} \sum_i p_i\cdot p^{\top}_i & \sum_i p_i \\ \sum_i p^{\top}_i & N \end{bmatrix}
 = \begin{bmatrix} Q_{p} & q \\ q^{\top} & N \end{bmatrix},
 \label{eq_q_decomposed}
\end{equation}
where $Q_p$ is the sum of outer products of 3D points and $q$ is the sum of all points, proportional to the mean. We consider the global reference frame of points for simplification, but each of the local coordinates at $t$ has been already transformed by $T_t$.

Then, the solution of the minimization problem is obtained directly by the eigendecomposition:
\begin{gather}
\pi^* = k \cdot u_{min}(Q(\bm{T})) \label{eq_eigen} \\
\text{s.t. }  ||\eta^*|| = 1. \nonumber
\end{gather}

The minimum eigenvector $u_{min}$ is proportional to the plane solution. The exact solution, in order to have geometrical meaning, must fulfill the definition of a plane (\ref{eq_plane_def}) with unit normal $\eta$. Therefore, it is multiplied by the scalar value $k$.
Without this correction, each of the errors would be scaled differently and the solution would be biased.
The summation of the point to plane squared error (for plane $\pi$) is equal to
\begin{equation}
    \lambda_{\pi}(\bm{T}) := \min_{\pi}  \pi^{\top} Q(\bm{T}) \pi = k^2 \lambda_{min}(Q(\bm{T})),
    \label{eq_ef_cost}
\end{equation}
where $\lambda_{min}$ is the minimum eigenvalue, which equals the minimum cost or minimum error after proper scaling by $k^2$, easily derived from (\ref{eq_eigen}).
Note that the minimization of the plane is implicit after solving the algebraic equation by the eigendecomposition, and as a consequence 
the only remaining state variables to be estimated are the trajectory ones, {\bf reducing the dimensionality} of the problem  considerably.
If considering the more general form of a bilevel problem:
\begin{gather*}
\min_{x,y} F(x,y) \\
\text{s.t.} \, \min_{y} f(x,y)
\end{gather*}
which can be viewed as the {\em lower-level} program $f(x,y)$ solving for plane parameters, while the {\em upper-level} program solves the trajectory \cite{colson2007overview}.

The function $\lambda_\pi(\cdot)$ obtained in (\ref{eq_ef_cost}) for the plane $\pi$ shows some interesting properties. In the ideal case of no-noise, this function should be zero and its values are always positive by construction $\lambda_\pi(\cdot)\geq 0$. We could reformulate this as an {\em implicit observation function} of the form $h(x,z)=0$, where in our particular case, the state $x$ is the trajectory $\bm{T}$ and the observation $z$ are all points. In contrast, an explicit observation is of the form $h(x)=z$, which relates the state with the observations.

\begin{figure*}[t!]
    \centering
    \includegraphics[width=.97\textwidth]{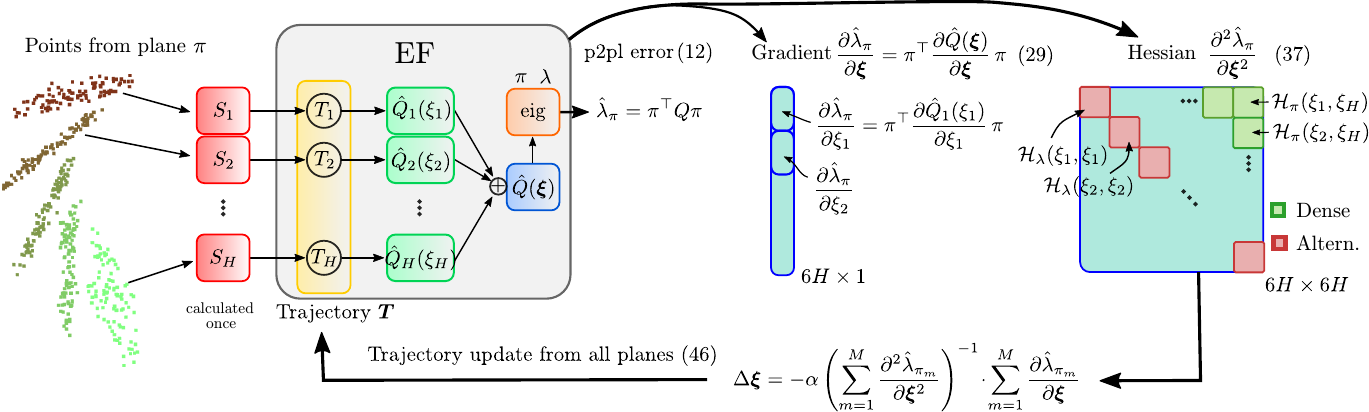}
    \caption{Diagram of the EF. From the left, the points from a single plane are depicted observed from different poses. The algorithm then estimates the $S_t$ matrices (only once) and process the data to estimate the plane $\pi$ and its eigenvalue $\lambda$. Following we can calculate the point error, the gradient (Sec.~\ref{sec_ef_diff}) and the Hessian (Sec.~\ref{sec_hessian}) in two forms, in red the {\em alternating} (block diagonal) and in green the {\em dense} method. This scheme corresponds to an iterative optimization process, so after updating the poses, the processes is repeated.}
    \label{fig_ef_flow}
\end{figure*}

Unfortunately, this implicit observation function is not suitable for general nonlinear least squares (NLLS) nor even simple Gauss-Newton.
The state estimation problem from $\lambda_\pi(\cdot)$, where the total cost to minimize is defined as

\begin{equation}
    C(\bm{T}) := \sum_{m=1}^M \lambda_{\pi_m}(\bm{T}),
    \label{eq_total_cost}
\end{equation}
given that $M$ is the total number of observed planes. Some important remarks:
\begin{itemize}
    \item The cost $C(\bm{T})$ is equal to the total point-to-plane squared error evaluated for each  of the poses $T_t$.
    \item The complexity is independent of the number of points, since the matrix $S_t$ requires them to be summed only once and used thereafter without modifications.
    \item The plane parameters are estimated in closed-form and the state variables are the sequence of poses $\bm{T}$ only.
\end{itemize}

In essence, the EF is a program that seeks to minimize the point to plane error w.r.t. the trajectory.
A diagram with the main parts can be seen in Fig.~\ref{fig_ef_flow}.
We will present as follows the specific functions and differentiation procedure to achieve this.

\section{Retraction and Optimization on the Manifold}\label{sec_manifold}

This section introduces some tools from optimization on matrix manifolds \cite{absil2009} required for the correct development of a second-order optimization method for the Eigen Factors. 
The problem of interest is to further develop the map $Q(T):\SE3 \to \mathbb{R}^{4\times 4}$, so we propose  systematic rules to differentiate any function involving $\SE3$ in a simple way.

A Rigid Body Transformation (RBT), denoted as  ${}^{f}T_{o} \in \SE3$, 
allows to transform points from an origin reference system $o$  to a final reference system $f$. It also transforms reference frames or simply expresses a 3D pose uniquely.
More formally, the Special Euclidean group is
\begin{equation}
SE(3) = \left\{ T =  \begin{bmatrix} \mathbf{R} & \bm{t} \\ 0  & 1 \end{bmatrix} \, | \, \mathbf{R} \in SO(3) \, \land \, \bm{t} \in \mathbb{R}^3 \right\}.
\end{equation}

In addition, RBT is a matrix group $\SE3 \subset \mathbb{R}^{4\times 4}$ and due to its inner structure and properties is a Riemannian manifold.
As such, there is a smooth mapping (chart) from a subgroup of $\SE3$ elements to the vector space $\mathbb{R}^6$.
It is well known that the dimensionality of a pose in 3D is six, however, there exist several possible mappings.

Our interest is on the optimization of a cost function whose input is an element of $\SE3$ or more generally, an element of a manifold $\mathcal{M}$.
Let the function $f(x):\mathcal{M} \to \mathbb{R}$ be a mapping from a manifold to a real value. Then, if one applies directly the definition of directional derivative
\begin{equation}
\nabla f(x) = \lim_{h \to 0} \frac{f(x+h\eta) - f(x)}{h},
\label{eq-directional-der}
\end{equation}
finds this operation not suitable for numerical calculations involving manifolds.

One can solve this problem by using a {\em retraction} mapping.
\begin{mthdef}
A retraction $R_T(\xi)$ is a smooth mapping from the tangent space around the $T$ element to the manifold: 
\begin{equation}
R_T(\xi):\mathbb{R}^6 \to \SE3.
\end{equation}
Two conditions must be satisfied: i) $R_T(0_T) = T$  and ii) local rigidity (see Ch.4 \cite{absil2009}).
\end{mthdef}

Accordingly, we can compose a new function when using a retraction:
\begin{equation}
\hat{f}_T = f \circ R_T,
\label{eq_func_retraction}
\end{equation}
from the tangent space $\mathbb{R}^6$ to a real value.
This new function $\hat{f}_T$ admits a directional derivative (\ref{eq-directional-der}). As a convention in this paper, all retractions are denoted with the superscript $\hat{\cdot}$ ({\em hat}).

As a result, an optimization problem $x' = x + \Delta x$, which is not well defined for elements of a manifold $x \in \mathcal{M}$,
becomes well defined when using a retraction
\begin{equation}
T' = R_T(\Delta \xi).
\end{equation}

In the case of first order optimization, the update is proportional to the gradient direction $\Delta \xi = -\alpha \nabla_{\xi}\hat{f}_T$. Gradient decent is one of the simplest methods for optimization: it requires a direction and a magnitude.
Second-order optimization methods, provide the direction of descent weighted by the curvature of the cost surface $\Delta \xi = -\alpha (\nabla_{\xi}^2\hat{f}_T)^{-1}\nabla_{\xi}\hat{f}_T$,
resulting in improved convergence rates.

%
In the particular case of $\SE3$, its tangent space has a structure of Euclidean space. Only under these conditions, the gradient and the Hessian are well defined after using a retraction:
\begin{gather*}
\nabla_T f(T) = \nabla_{\xi}\hat{f}_T(0_T) \\
\nabla_T^2 f(T) = \nabla_{\xi}^2\hat{f}_T(0_T).
\end{gather*}

These results are significant since they will allow to use the concept of directional derivatives (\ref{eq-directional-der}) by simply applying real analysis into the calculation of the gradient and Hessian in Sec.\ref{sec_ef_diff}.
By doing that, we will implement a second-order optimization method that have almost no need for hyper-parameters and quadratic convergence rate \cite{nocedal2006}.

\subsection{Differentiating over the Exponential Map in $\SE3$}

The exponential map is the most natural retraction choice: its definition is highly related with the concept of perturbations over $\SE3$ which will facilitate the calculation of derivatives in the following sections.
There are other options for choosing a retraction mapping, as long as they satisfy the definition. 

The exponential map defines the retraction $R_T$ in the following way:
\begin{equation}
R_T(\xi) = \Exp(\xi)T.
\label{eq_retraction}
\end{equation}

We follow a left-hand-side convention, but the right-hand-side is a valid retraction and it would require a similar derivation of the terms.
The retraction around the identity can be expanded by using the matrix exponential definition as a Taylor series
\begin{equation}
R_I(\xi) = \Exp(\xi)I = I + \xi^{\land} + \frac{1}{2}(\xi^\land)^2 + o(|\xi|^3),
\label{eq_exp_identity}
\end{equation}
where the vector $\xi = [ \theta\tr ,\rho\tr]\tr = [\xi_1,\ldots ,\xi_6]\tr \in \mathbb{R}^6$ and the matrix
\begin{eqnarray}
\xi^{\land} =  \begin{bmatrix} \ohat & \rho \\ 0 & 0 \end{bmatrix} = 
           \begin{bmatrix} 0 & -\theta_3 & \theta_2 & \rho_1\\ 
                              \theta_3 & 0 & -\theta_1 & \rho_2\\
                              -\theta_2 & \theta_1 & 0 & \rho_3 \\
                              0 & 0 & 0 & 0
\end{bmatrix}	
\label{eq_lie_generators}
\end{eqnarray}
is the Lie Algebra matrix group $\mathfrak{se}(3)$ that represents the group of infinitesimal RBT around  the identity. See \cite{lynch2017,barfoot2017,sola2018micro} for a comprehensive explanation on the topic of RBT and its Lie algebra.
This matrix can be rearranged such that it has a vector space structure
\begin{equation}
\xi^{\land} = G_1 \xi_1 + G_2 \xi_2 + \ldots + G_6 \xi_6,
\end{equation}
where $G_i$ is a $4\times 4$ matrix of the i-$th$ generator
and the basis $\bm{G} = \{G_i\}, \forall i = 1,\ldots , 6$.


After defining the retraction map (\ref{eq_retraction}), the next step is to obtain the derivative of the retraction, i.e., the exponential map evaluated at the zero element.
\begin{lem}
The derivative of the exponential map with respect to one of the Lie coordinates $\xi_i$ is its corresponding matrix generator
\begin{equation}
\frac{\partial \Exp(\xi)}{\partial \xi_i} = G_i.
\end{equation}
\end{lem}
\begin{proof}
One can expand analytically the definition of the matrix exponent in (\ref{eq_exp_identity})
\begin{align}
\frac{\partial \Exp(\xi)}{\partial \xi_i} & = \frac{\partial}{\partial \xi_i}\Big(I + \xi^{\land} + \frac{1}{2}(\xi^\land)^2 + o(|\xi|^3)\Big)\bigg|_{\xi_i = 0} \nonumber \\
 & = G_i + \frac{1}{2}G_i\xi^{\land} + \frac{1}{2}\xi^\land G_i + \ldots \bigg|_{\xi_i = 0} = G_i
\end{align}
where after differentiating each of the terms, the higher order terms vanish when evaluated at $\xi_i=0$.
\end{proof}

\begin{lem}
The second derivative of the exponential with respect to the Lie coordinates $\xi_i$ and $\xi_j$ is
\begin{equation}
    \mathcal{H}_{ij} := \frac{\partial^2 \Exp(\xi)}{\partial \xi_j \partial \xi_i} = \frac{1}{2}G_iG_j +  \frac{1}{2}G_jG_i.
\end{equation}
\end{lem}

\begin{proof}
one can repeat the same procedure to expand analytically the second-order derivative:
\begin{align}
\frac{\partial^2 \Exp(\xi)}{\partial \xi_j \partial \xi_i} & = 
\frac{\partial^2}{\partial \xi_j \partial \xi_i}\Big(I + \xi^{\land} + \frac{1}{2}(\xi^\land)^2 + o(|\xi|^3)\Big)\bigg|_{\xi = 0} \nonumber \\
& = \frac{\partial}{\partial \xi_j}\Big(G_i + \frac{1}{2}G_i\xi^{\land} + \frac{1}{2}\xi^\land G_i + o(|\xi|^3) \Big)\bigg|_{\xi = 0} \nonumber \\
& = \frac{1}{2}G_iG_j +  \frac{1}{2}G_jG_i
\end{align}
\end{proof}

This compact result is a $4 \times 4$ matrix, one for each of the $(i,j)$ coordinates of the Hessian.
The overall tensor $\frac{\partial^2 \Exp(\xi)}{\partial \xi^2}$ 
has dimensions $4\times 4 \times 6 \times 6$ when considering each of its coordinates $(i,j)$. Note the symmetry in $\mathcal{H}_{ij} = \mathcal{H}_{ji}$.

The idea of retraction functions plus the first and second-order derivatives of the exponential map will be enormously useful when calculating derivatives of more complex functions, such as the matrices in the EF, as we show in the next section.

\section{Eigen-Factors Differentiation}\label{sec_ef_diff}


Let's start considering only the transformation $T_t$ at time $t$, then $Q=Q_t$ 
can be rewritten as the retraction function (\ref{eq_func_retraction})  of the form
$\hat{Q}_{t} = Q_t \circ R_{T_t}:\mathbb{R}^6 \to \mathbb{R}^{4\times 4}$, such that
\begin{equation}
\hat{Q}_{t}(\xi_t) = \Exp(\xi_t)\cdot Q_t \cdot \Exp(\xi_t)^{\top}.
\label{eq_q_t}
\end{equation}

This form generalizes to the full trajectory, since the $Q$ matrix in (\ref{eq_q}) is defined as a sum of $Q_t$ and it is a function of the trajectory $\bm{T}$. Accordingly, it is also a function of the retraction coordinates $\bm{\xi} = \{\xi_{1,1}, \xi_{1,2} \ldots, \xi_{t,i},\ldots \xi_{H,6}\}$, which are the variables that support the optimization. We are using Lie coordinates $i$, and trajectory elements $t$ so in total the dimension is $6\times H$.

The point to plane error, defined as the $\lambda_\pi(\bm{T})$ in the EF derivation in (\ref{eq_ef_cost}), can be as well redefined as the retraction function 
\begin{equation}    
\hat{\lambda}_\pi(\xi_t) = \lambda_\pi \circ \hat{Q}_{t}: \mathbb{R}^6 \to \mathbb{R}
\end{equation}
which is now a well defined function to calculate directional derivatives (\ref{eq-directional-der}). One can evaluate the retraction around any other pose in a more general form $\hat{\lambda}_\pi(\bm{\xi})$.

Now, it is necessary to propagate the gradient through the eigendecomposition.

\begin{thm}
The derivative of the eigenvalue $\lambda_\pi$ with respect to the pose $T_t$ equals
\begin{equation}
\nabla_{\xi_t}\hat{\lambda}_\pi = \frac{\partial \hat{\lambda}_\pi}{\partial \xi_t} = 
\pi^{\top} \frac{\partial \hat{Q}_t}{\partial \xi_t} \, \pi.
\label{eq_eig_gradient}
\end{equation}
\end{thm}
\begin{proof}
By definition, the eigendecomposition is expressed as $\hat{Q} u = \hat{\lambda} u$.
The vector of parameters $u$ is a unit vector s.t. $||u||^2=u^{\top}u=1$.
One can exchange the unit vector $u$ for any other vector $\pi \in \mathbb{P}^3$, since $\pi = k u$ and $||\pi|| = k^2$,
from definition $\hat{\lambda}_\pi$ in (\ref{eq_ef_cost}).

$\hat{Q} \pi = \hat{\lambda} \pi$ is derived wrt each of $i$th coordinates of the variable $\xi_{t,i}$, for simplicity it is written as $\xi_{i}$, 
\begin{align}
\label{eq_diff_svd}
\hat{Q} \frac{\partial \pi}{\partial \xi_i} + \frac{\partial \hat{Q}}{\partial \xi_i} \, \pi = \lambda \,\frac{\partial \pi}{\partial\xi_i}  + \frac{\partial \hat{\lambda}}{\partial\xi_i} \, \pi  \\
\text{s.t.} \quad \pi^{\top} \,\frac{\partial \pi}{\partial\xi_i} = 0.
\label{eq_diff_svd_v}
\end{align}

Taking into account that the matrix $\hat{Q}$ is symmetric by construction, then
\begin{equation}
\hat{Q}=\hat{Q}^{\top} \quad \implies \quad \pi^{\top}\hat{Q} = \hat{\lambda} \pi^{\top}.
\label{eq_dual}
\end{equation}

One can pre-multiply the expression (\ref{eq_diff_svd}) by $\pi^{\top}$, substitute (\ref{eq_dual}) and do some manipulations:

\begin{align}
 \pi^{\top} \hat{Q} \,\frac{\partial \pi}{\partial \xi_i} + \pi^{\top} \frac{\partial \hat{Q}}{\partial \xi_i} \, \pi &= \pi^{\top}\hat{\lambda} \,\frac{\partial \pi}{ \partial \xi_i}  +\pi^{\top} \frac{\partial \hat{\lambda}}{\partial \xi_i} \, \pi    \nonumber \\
\hat{\lambda} \underbrace{\pi^{\top}\,\frac{\partial \pi}{\partial\xi_i}}_{(\ref{eq_diff_svd_v})} + \pi^{\top} \frac{\partial \hat{Q}}{\partial\xi_i} \, \pi &= \hat{\lambda} \underbrace{\pi^{\top}\frac{\partial \pi}{\partial\xi_i}}_{(\ref{eq_diff_svd_v})}    + \, \frac{\partial \hat{\lambda}}{\partial\xi_i} \underbrace{\pi^{\top} \pi}_{k^2}    \nonumber \\
\pi^{\top} \frac{\partial \hat{Q}}{\partial\xi_i} \, \pi &= \frac{\partial \hat{\lambda} k^2}{\partial\xi_i} = \frac{\partial \hat{\lambda}_\pi}{\partial\xi_i}.
\label{eq_diff_derivation}
\end{align}

Recall from (\ref{eq_q}) that $\hat{Q} = \sum \hat{Q}_\tau$, so the only derivative that is not zero is the one corresponding to the $\hat{Q}_t$ component. One can re-arrange each of the $i$ coordinates in $\xi_t$ and finally write 
\begin{equation}
\frac{\partial \hat{\lambda}_\pi}{\partial\xi_t} =  \pi^{\top} \frac{\partial \hat{Q}_t}{\partial\xi_t} \, \pi.
\end{equation}
\end{proof}

Accordingly, the calculation of the gradient boils down to calculate the derivative of $\hat{Q}_t$ with respect to the tangent space around the element $T_t$ of the manifold. We   differentiate (\ref{eq_q_t}) such that

\begin{equation}
\frac{\partial \hat{Q}_{t}}{\partial \xi_t} = \frac{\partial\text{Exp}(\xi_t)}{\partial \xi_t} Q_t + Q_t \frac{\partial\text{Exp}(\xi_t)^{\top}}{\partial \xi_t},
\label{eq_a_left}
\end{equation}
which is obtained by applying the product rule. 

By using Lemma 1, one simply substitutes to obtain the derivative wrt the coordinate $i$:
\begin{equation}
\frac{\partial \hat{Q}_{t}}{\partial \xi_{t,i}} = G_i \cdot Q_t + Q_t \cdot G_i^{\top},
\label{eq_grad}
\end{equation}
where each of these elements is a  $4\times 4$ matrix. The gradient is then obtained after multiplying this result as in (\ref{eq_eig_gradient}) which yields a scalar value; in total a 6D vector for each of the $i=1,\ldots, 6$ coordinates.

This methodology is general for obtaining derivatives of functions w.r.t. poses in 3D and allows to have a compact result, even when the image of this function is a matrix.

\subsection{Eigen Factors, Hessian}\label{sec_hessian}

The objective of this section is to obtain an analytical solution for the Hessian matrix $\nabla_{\xi}^2\hat{\lambda}$. 
Similarly, we  follow the same systematic procedure to derive the Hessian of the EF.

\begin{thm}
The Hessian of the Eigendecomposition with respect to the poses $T_t$ and $T_{t'}$ is
\begin{align}
    \frac{\partial^2 \hat{\lambda}_{\pi}}{\partial \xi_{t,j} \partial \xi_{t',i} } =& \,2 \cdot \mathcal{H}_{\lambda}(\xi_{t,j}, \xi_{t',i}) + 2 \cdot \mathcal{H}_{\pi}({\xi_{t,j}, \xi_{t',i}}) \label{eq_hessian}\\
  \mathcal{H}_{\lambda}(\xi_{t,j}, \xi_{t,i}) =& \, \pi^\top \mathcal{H}_{ij} Q_t \pi  + \pi^\top G_{t,j} \frac{\partial  \hat{Q_{t}}}{\partial \xi_{t,i}} \pi  \label{eq_hessian_diag}\\
    \mathcal{H}_{\lambda}(\xi_{t,j}, \xi_{t',i}) =& \, 0 \qquad \forall t\neq t' \label{eq_hessian_zero}\\
   \mathcal{H}_{\pi}({\xi_{t,j}, \xi_{t',i}}) =& \, \pi^{\top} \frac{\partial  \hat{Q_t}}{\partial \xi_{t,j}} \mathcal{Q}^{-1} \frac{\partial  \hat{Q}_{t'}}{\partial \xi_{t',i}} \pi  \label{eq_hessian_dense}   \\
  \mathcal{Q}^{-1} =&  \sum_{l\neq min}^4 \frac{1}{\lambda_{min}- \lambda_l}u_l u_l^{\top}  \label{eq_pseudo_q_inv}.
\end{align}
\end{thm}
\begin{proof}
Analytical derivation is shown in Appendix \ref{app_hessian}.
\end{proof}

There are two main components in the EF Hessian expression (\ref{eq_hessian}). From the one hand, the component corresponding to the eigenvalue $\mathcal{H}_{\lambda}(\xi_{t,j}, \xi_{t,i})$, whose elements are zero for different time indexes (\ref{eq_hessian_zero}).
On the other hand, the component related to the derivative of the plane $\mathcal{H}_{\pi}({\xi_{t,j}, \xi_{t',i}})$. This second term yields a dense hessian matrix. The vector $u_l$ is the eigenvector of $Q$ as introduced in (\ref{eq_homogeneous_estimate}) for the $u_{min}$.
Next subsection will motivate the contribution of each of these components and its usage.

\subsection{Alternating Optimization}\label{sec_ef_altern}

We have presented EF as a method that first estimates planes in closed form (\ref{eq_eigen}) and in a second step EF optimizes the trajectory $\bm{T}$, effectively reducing the dimensionality of the state to be estimated, as a bilevel program.

In an {\em alternating optimization} scheme, the two set of variables, the the {\em upper} and {\em lower} level are separable, i.e. the two set of variables is not related to each other. Unfortunately, the derived solution for the trajectory shows a dependency in the Hessian w.r.t the derivatives of planes.
Therefore we propose the following approximation
\begin{equation*}
\frac{\partial \pi}{\partial \bm{T}} \approx 0.
\label{eq_alternating_deriv}
\end{equation*}

By doing this, the gradient remains the unaltered (\ref{eq_eig_gradient}) and the updated Hessian (\ref{eq_hessian}) is

\begin{equation}
\frac{\partial^2 \hat{\lambda}_{\pi}}{\partial \xi_{t,j} \partial \xi_{t,i} } \approx \,2 \cdot \mathcal{H}_{\lambda}(\xi_{t,j}, \xi_{t,i}).
\label{eq_alternating_hessian}
\end{equation}

An important property derived from this approximation, as shown in (\ref{eq_hessian_zero}),  is that the Hessian matrix is {\bf block diagonal}, which implies a fast calculation of the optimization iterations and a reduced complexity of the algorithm. 

The motivation for such an approximation is on the inspection of the matrix
$\mathcal{Q}^{-1}$. One can notice it depends roughly on the inverse of points $1/N$, as stated in  (\ref{eq_q}). So with more points and poses, the contribution of the gradient w.r.t. the plane diminishes.

This block-diagonal form could also be viewed as a quasi-Newton method \cite{nocedal2006}, which arises for the particular problem of the EF, allowing for an estimate of the optimization step size, with superior convergence rate than first order optimization methods, but less exhaustive than pure second order methods with a dense Hessian.

\subsection{Trajectory Optimization}\label{sec_ef_traj_optm}

In this section, we discuss different options for optimizing a trajectory $\bm{T} = \{T_1, \ldots,T_H\}$ or its corresponding set of vectors $\bm{\xi}$.
One can directly apply the EF gradient and Hessian at each pose in the trajectory, resulting in the joint gradient vector
\begin{equation}
\frac{\partial \hat{\lambda}_\pi }{\partial \bm{\xi}} = 
  \left[ \Big( \pi^{\top} \frac{\partial \hat{Q}}{\partial \xi_1} \pi \Big) ^{\top},
  \ldots, \Big(\pi^{\top}\frac{\partial \hat{Q}}{\partial \xi_H} \pi \Big) ^{\top} \right]^{\top}_{6H\times 1}.
\label{eq_gradient}
\end{equation}
This aggregation corresponds to the Cartesian product of the manifolds of each of the poses. In case there were not points observed from a pose, then the gradient would zero since it is not contributing to the EF cost in any mean.
The same arrangement can be done for the Hessian:
\begin{equation}
\frac{\partial^2 \hat{\lambda}_\pi }{\partial \bm{\xi}^2} = 
  \left[ \frac{\partial^2 \hat{\lambda}_{\pi}}{\partial \xi_{t,j} \partial \xi_{t',i} } \right]_{6H\times 6H}.
\label{eq_hess}
\end{equation}

The total EF cost (\ref{eq_total_cost}) can be redefined in terms of retraction function w.r.t. the joint vector of poses $\bm{\xi}$:
\begin{equation}
    \hat{C}(\bm{\xi}) = \sum_{m=1}^M \hat{\lambda}_{\pi_m}(\bm{\xi}).
    \label{eq_total_cost_ret}
\end{equation}

The optimal update of the joint vector of pose coordinates is solved by 
\begin{gather}
    \Delta \bm{\xi} = -\alpha (\nabla_{\xi}^2 \hat{C})^{-1}\nabla_{\xi}\hat{C} \\
    \mathcal{G}:=\nabla_{\xi}\hat{C} = \sum_{m}^M \frac{\partial \hat{\lambda}_{\pi_m} }{\partial \bm{\xi}} \\
    \mathcal{H} :=\nabla_{\xi}^2 \hat{C} = \sum_{m}^M \frac{\partial^2 \hat{\lambda}_{\pi_m} }{\partial \bm{\xi}^2}.
\end{gather}

This problem is a standard second order method, after applying the retraction process described above. Note it is an optimization on the manifold of $\SE3$.

\subsection{Error Invariance to the Reference Frame}

Planes are elements of the projective space $\mathbb{P}^3$  and they can be transformed from one reference frame to another, similarly to homogeneous points by RBT:
\begin{equation}
{}^{t}\pi =  {}^{t}T_g^{-\top} \cdot {}^{g}\pi,
\label{eq_plane_transf}
\end{equation}
where the plane ${}^{g}\pi$ in the global reference frame is transformed into ${}^{t}\pi$ in the local frame $t$.

If we account explicitly  all reference frames in the point to plane error in (\ref{eq_ef_plane_estimation}), where the point ${}^{t}p$ is expressed in a local coordinate and the plane is expressed in the global coordinate $g$, then
\begin{equation}
{}^{g}\pi^{\top} \cdot {}^{g}T_t \cdot {}^{t}\tilde{p} = {}^{g}\pi^{\top} \cdot {}^{g}\tilde{p} = {}^{t}\pi^{\top} \cdot {}^{t}\tilde{p}.
\label{eq_plane_transf_equiv}
\end{equation}

One can see that the error value is invariant to the choice of the reference: we could either transform any of the elements to the local reference frame or to the global frame. The EF development used the latter.

Not just that, but the planar constraint is fulfilled in any reference frame:
\begin{equation}
{}^{c}\pi^{\top} \cdot {}^{c}\tilde{p} = ( {}^{c}T_g^{-\top} {}^{g}\pi )^{\top} \cdot  {}^{c}T_g {}^{g}\tilde{p} = {}^{g}\pi^{\top} \cdot {}^{g}\tilde{p},
\label{eq_plane_equiv}
\end{equation}
where the transformations cancel out for any new reference frame $c$. The next subsection will delve deeper into this idea.

\subsection{Data Centering}\label{sec_data_center}

Since the plane constraints are of equal value for any reference frame ${}^{c}T_g$ (for simplicity we write $T_c$) we will show that the EF cost is also {\bf invariant} to transformations:

\begin{align}
    \lambda_{\pi} &=
    \pi^{\top} {Q} \pi = \pi^{\top} \underbrace{T_c^{-1} T_c}_{I}  \cdot {Q} \cdot T_c^\top  \underbrace{T_c^{-\top}\pi}_{{}^{c}\pi} \nonumber \\
    & = {}^{c}\pi^{\top} T_c  \cdot {Q} \cdot T_c^\top {}^{c}\pi.
\end{align}

The plane coordinates have changed, and the matrix ${Q}$ is also transformed into the reference frame $c$, by looking at the definition in (\ref{eq_ef_plane_estimation}) and (\ref{eq_q}).

If the error $\lambda_{\pi}$ remains constant, then one can naturally ask: Is there a reference frame more convenient than others to calculate the error? In order to answer this question, we have to analyze in more detail the elements of the matrix $Q$ (\ref{eq_q_decomposed}): $Q_p$ and $q$.
There is a relation between $q$ and the mean $\mu$ of all points from the plane, recall observed in the same global reference frame
\begin{equation}
    q = \sum_t^H \sum_n^{N_t}  T_t \cdot p_{t,n} = \sum_t^H  q_t = N \cdot \mu.
    \label{eq_q_summation}
\end{equation}
The vector $q$ is the summation of $q_t$ each the average point in global coordinates.

Similarly, the  $3\times 3$ matrix block $Q_p$ is a summation of points transformed into the same reference frame:
\begin{equation}
    Q_p = \sum_t^H \sum_n^{N_t} T_t \cdot p_{t,n} \cdot p_{t,n}^{\top} \cdot T_t^{\top}
    \label{eq_Qp_summation}
\end{equation}
which is related to the point covariance $\Sigma_p$ in (\ref{eq_center_estimate}) in the following manner:

\begin{equation}
\Sigma_p \cdot N = Q_{p} - \frac{1}{N} q \cdot q^{\top} .
\label{eq_q_cov}
\end{equation}

Now, let the transformation be a translation by a the mean value:
\begin{equation}
 T_c = \begin{bmatrix} I & -\mu \\ 0 & 1 \end{bmatrix}.
\end{equation}

In this case, the result of this transformation naturally yields the centering of the data points:
\begin{equation}
Q_c := T_c \begin{bmatrix} Q_{p} & q \\ q^{\top} & N \end{bmatrix} T_c^\top  = 
 \begin{bmatrix} Q_{p} - \frac{1}{N}q\cdot q^{\top} & 0 \\ 0 & N \end{bmatrix},
\label{eq_Q_centered}
\end{equation}
where top left block corresponds to the scaled  covariance (\ref{eq_q_cov}).

Alternatively, we could have asked: what is the transformation that applied to the plane $\pi$ yields a $d$ component zero?
Then, the solution is
\begin{equation}
\begin{bmatrix} \eta \\ d \end{bmatrix} = \begin{bmatrix} I & 0 \\ -\mu^{\top} & 1 \end{bmatrix} \begin{bmatrix} \eta \\ 0 \end{bmatrix},
\label{eq_center_plane_vector}
\end{equation}
which is valid since we already defined $d = - E\{\eta^{\top} p\}$ in (\ref{eq_center_estimate}).
In both cases, the result is the centered matrix $Q_c$. 

In practice, we have used this result to estimate the plane normal, since it provides more stable results as it will be discussed in Sec.\ref{sec_exp_numerical}. The development is however taken from the homogeneous formulation presented above, one just needs to adjust to this {\em transformation} in the gradient and Hessian calculations. Accordingly, the gradient is calculated as in (\ref{eq_eig_gradient}), and almost all components of the Hessian remain the same except for (\ref{eq_hessian_dense}) now written as

\begin{align}
   \mathcal{H}_{\pi,c}({\xi_{t,j}, \xi_{t',i}}) =& \, \pi^{\top} \frac{\partial  \hat{Q_t}}{\partial \xi_{t,j}} T_c^{\top} \mathcal{Q}_c^{-1} T_c \frac{\partial  \hat{Q}_{t'}}{\partial \xi_{t',i}} \pi \label{eq_hessian_plane_center} \\
   \mathcal{Q}_c^{-1} =& \begin{bmatrix} \sum_{l\neq min}^3 \frac{1}{\lambda_{min}- \lambda_l}v_l v_l^{\top} & 0 \\ 0 & \frac{1}{(\lambda_{min}-N)}
    \end{bmatrix}.
\label{eq_hessian_Q_plane_center}
\end{align}

The variables $v_l$ are the eigenvectors of the $3\times 3$ block matrix in (\ref{eq_Q_centered}) and $v_{min}$ corresponds to the plane normal $\eta$.
In Appendix \ref{app_center} is shown this development.

In the evaluations, we will use the method just described in this section, the EF-center, which is referred simply as EF.

\section{Experiments}\label{sec_exps}
Experiments are targeted to evaluate performance of different plane-based SLAM \textbf{back-end} approaches, formulated in the form of graph SLAM. Since the quality of the SLAM full pipeline could also be affected significantly by detection and association algorithms on the front-end, the question about quality of the full pipeline remains out of the scope of this paper and instead we focus only on the back-end. To achieve this, we consider pre-labeled data from different sensor modalities by developing our synthetic generator and using labeled data for ICL NUIM \cite{handa2014benchmark}
and KITTI \cite{geiger2013} sequences from the EVOPS benchmark~\cite{kornilova2022}.

\begin{figure*}[t!]
    \centering
    \includegraphics[width=.97\textwidth]{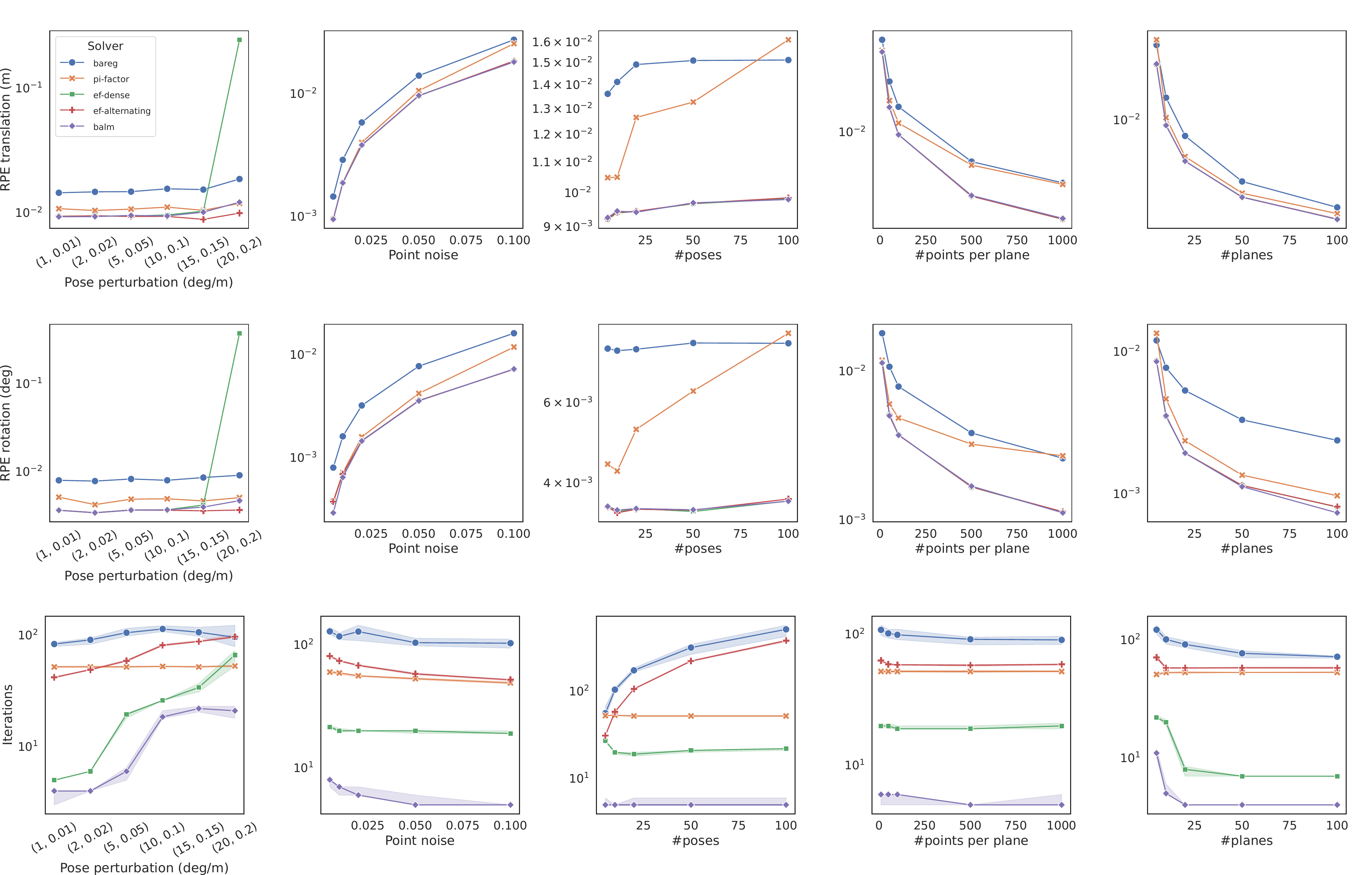}
    \caption{Evaluation results in synthetic environments with planes sweeping over such parameters as: pose perturbation, number of poses in trajectory, number of points per plane, point-to-plane noise. \emph{Top}: RPE translation error in different simulator configurations, \emph{Middle}: RPE rotation error in different simulator configurations, {\em Down}: Number of iterations to converge. Plotting shaded regions for the CI of 95\% (default).}
    \label{synth_eval}
\end{figure*}


The following approaches are taken for evaluation as the main representatives of methods on graph plane-based SLAM: BAREG~\cite{huang2021}, $\pi$-Factor~\cite{zhou2021lidar}, BALM~\cite{liu2021balm, liu2022efficient}, EigenFactor alternating (simply EF) and EF-Dense, last two are our proposed methods. BAREG, $\pi$-Factor and EigenFactor are implemented in our common framework \texttt{mrob}\footnote{\texttt{https://github.com/prime-slam/mrob}} for factor-graph optimization problems with python bindings and shared open-source for the community. In particular, we use a Levenberg-Marquadt optimizer with relative tolerance of $10^{-4}$ for all the methods and the exact same initial conditions. Since the BALM formulation is also graph SLAM, we consider the author's original implementation with default parameters~\cite{liu2022efficient}, removing the front-end part of the algorithm for a fair comparison with other methods for back-end evaluation.

\subsection{Numerical Evaluations on EF}\label{sec_exp_numerical}

Numerical evaluations provide empirical support for the choices and correct implementation of the Eigen Factors.
In total, we conduct two numerical tests: plane estimation and gradient/Hessian numerical accuracy.

The {\bf plane estimation} tests two different methods: centered (\ref{eq_center_estimate}) and homogeneous (\ref{eq_homogeneous_estimate}).
On mild conditions of noise and distance to the reference origin, both methods perform similarly. However, after adding noise and translating the data points further away from the origin, the {\em centered} method provides superior results without doubt. This is an effect of the eigendecomposition and the numerical noise introduced to the unit vector solution (variable), which is scaled by $k$ (\ref{eq_eigen}).

The {\bf numerical tests} of the gradient and Hessian confirm the correct implementation of the analytical derivations (\ref{eq_eig_gradient}), (\ref{eq_hessian}) and (\ref{eq_hessian_plane_center}). We have compared the obtained values with the numerical ones and we use percentage of the norm. For instance, for the Hessian:

\begin{equation}
 \frac{||\mathcal{H}-\mathcal{H}_{numerical}||}{||\mathcal{H}||} \times 100.
\end{equation}

For trajectories with small perturbations, the percentage norm is close to zero for both the gradient and the Hessian, showing empirically the correctness of the derivation.

\begin{figure}[]
    \centering
    \includegraphics[width=.35\textwidth]{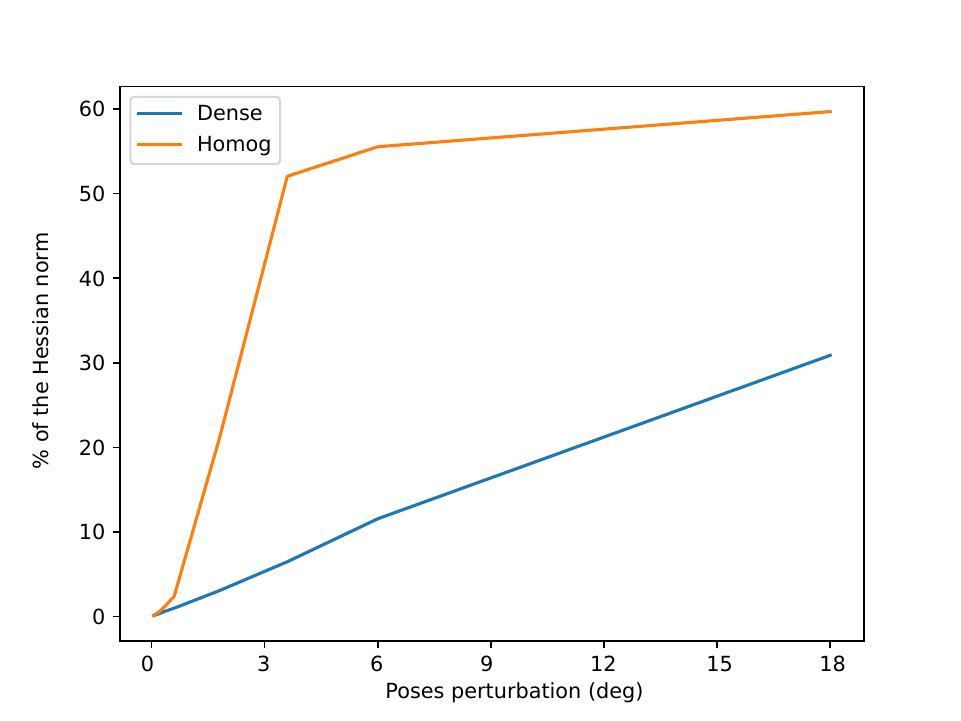}
    \caption{Hessian error, percentage of the norm after injecting noise to the trajectory. This result is for 6 planes, 10 poses and 125 points/plane.}
    \label{fig_hessian_error}
\end{figure}

When evaluating after adding more noise, the accuracy degrades, however it does differently for the Hessian using the homogeneous $\mathcal{Q}$ in (\ref{eq_pseudo_q_inv}) than the center approach $\mathcal{Q}_c$ in (\ref{eq_hessian_Q_plane_center}), as it can be seen in Fig.~\ref{fig_hessian_error}. Therefore, on the following analysis we have chosen the center approach as the default EF method.

\subsection{Complexity Analysis}

The theoretical complexity analysis is depicted in Table~\ref{method_complexity}. We have included three categories that are required at each iteration step of the optimization, assuming that planes are observed over the full sequence.

\begin{table}[h]
\caption{Complexity comparison of various plane SLAM methods.}
\label{method_complexity}
\begin{tabular}{|c||c|c|c|}
\hline
               & \# Pl. Estim. & Evaluation  & Optimization \\
\hline
BAREG            & $MH$ & $O(HM)$  & $O(H)$ \\
\hline
$\pi$-Factor           & 0 &  $O(HM)$  & $O(H^3+M)^*$ \\
\hline
BALM2          & $M$ & $O(HM)$  & $O(H^3)$ \\
\hline
EF-Dense          & $M$ & $O(HM)$  & $O(H^3)$ \\
\hline
EF          & $M$  & $O(HM)$  & $O(H)$ \\
\hline
\end{tabular}
\end{table}

First column is the number of planes to be estimated. For instance, $\pi$-Factor does not evaluate the planes, since they are part of their state variables. Other methods evaluate each plane once and BAREG needs to evaluate each plane at each pose in the trajectory.
Second column~--- {\em evaluation} refers to the complexity to evaluate the residuals for each method and calculate the gradient and Hessian if needed.
We see how all methods provide the same complexity which depends on the number of planes $M$ and poses in the trajectory $H$.

Finally, the optimization complexity is related to the structure of the problem. For instance BAREG and EF show linear complexity since they have a block diagonal matrix to invert. $\pi$-Factor estimates a larger state vector $H+M$ and requires to invert a sparse matrix (we show worst case) and Schur trick could be used to marginalize the plane landmarks $M$.
BALM2 and EF-dense require to invert a dense matrix, so they show cubic complexity.
Next section will report the runtime efficiency of each of the methods for comparison.

\subsection{Synthetic planes}

In order to evaluate the quality of planar-based SLAM back-end, we have developed a generator of synthetic environment with planes. 

\textbf{Setup.} The generator provides a set of point clouds with planes and a reference trajectory for them, and could be parameterized by number of poses, number of planes, number of points per plane, point-to-plane noise, trajectory noise. Example of generated synthetic environment is presented in Fig.~\ref{fig_syntehtic_env}. The default parameters configuration is the following: number of poses~--- 10, number of planes~--- 10, number of points per plane~--- 50, point-to-plane noise~--- 0.04 meters, trajectory perturbation per each pose~--- 0.05 meters and 5 degrees. In evaluation, we choose one parameter to be swept over different values whereas other parameters remain constant. Relative Pose Error~\cite{kummerle2009rpe} is used to estimate the accuracy of methods, rotation and translation part independently.

\begin{figure}[h]
    \centering
    \includegraphics[width=.46\textwidth]{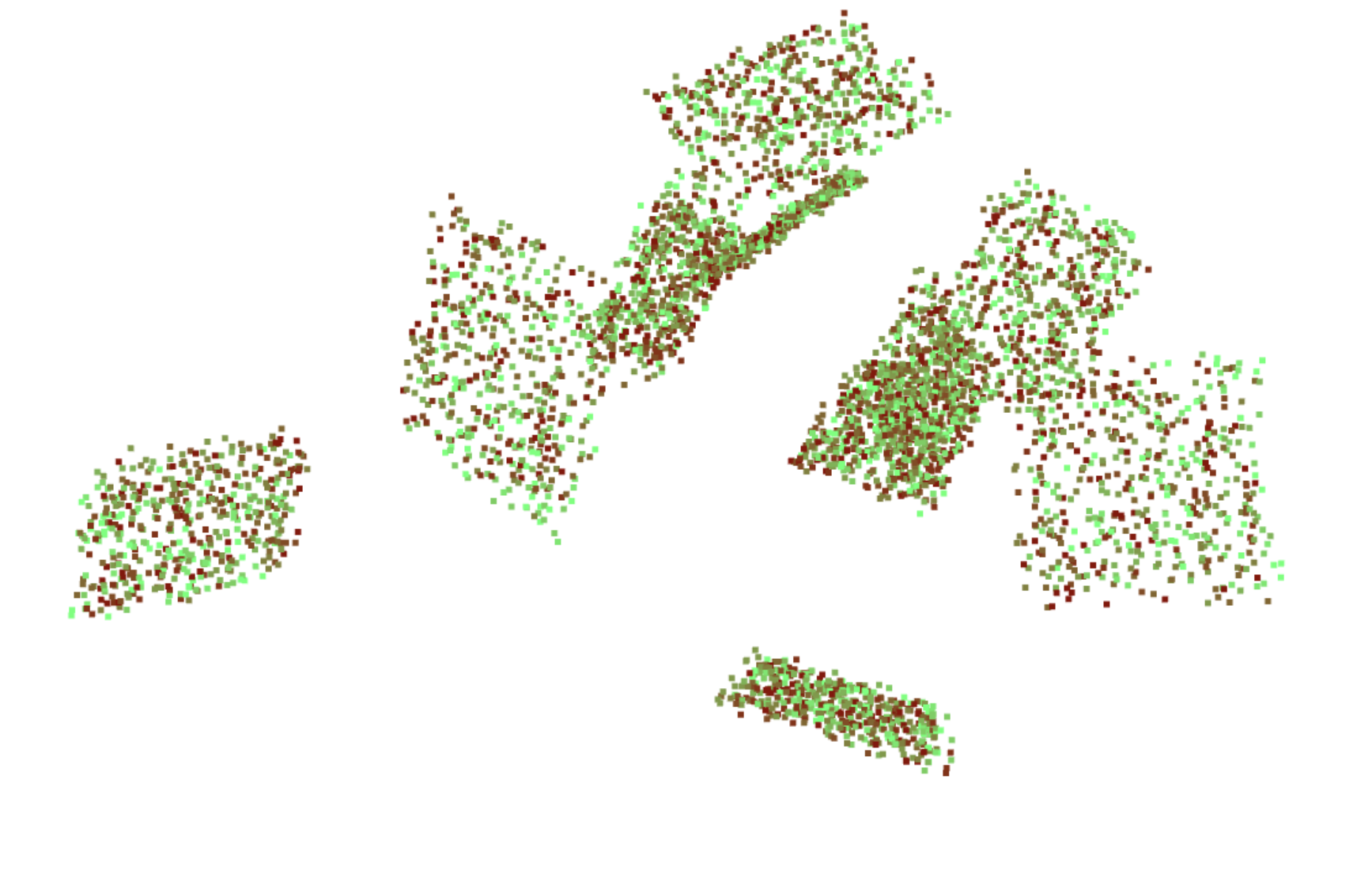}
    \caption{Synthetic Environment. Points are generated to emulate sampling over planes. Color corresponds to different pose.}
    \label{fig_syntehtic_env}
\end{figure}

\textbf{Results and discussion.} Results of evaluation are depicted in Fig.~\ref{synth_eval}. After parameters variations, the method could be broadly ordered as follows: EF (alternating), EF-Dense and BAREG show almost the same quality in RPE translation and rotation, followed closely by $\pi$-Factor. BAREG performs slightly worse under these conditions.

It is worth to note that in the trajectory perturbation experiment, EF, $\pi$-Factor and BALM show consistent results on all perturbation ranges, meaning that they show good convergence properties and could be applied for global alignment, not only local trajectory refinement. Sweeping over parameters such as number of poses, number of points per plane, point-to-plane noise, the method's quality is ranged in the same order. When considering number of iterations, BALM provides best results, although as with EF-Dense, each iteration requires  more computational resources (Table \ref{method_complexity}). EF requires more iterations at a lower computational complexity.

Figure \ref{fig_synthetic_runtime} reports the runtime results for the methods implemented under the same framework. We can observe how the total time per optimization provides a common ground for all methods, performing similarly, while the number of iterations is not a good indicator for efficiency.

\begin{figure}[h]
    \centering
    \includegraphics[width=.23\textwidth]{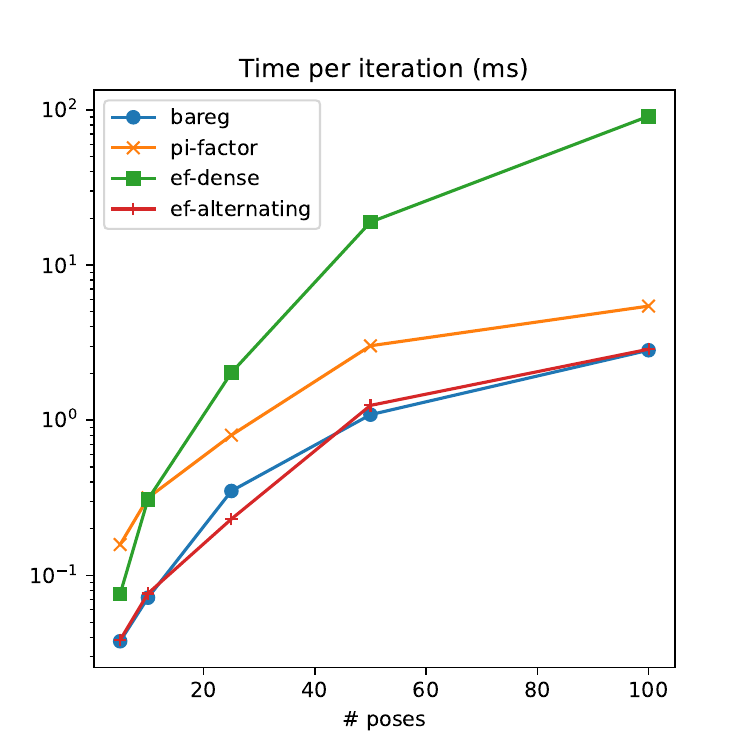}
    \includegraphics[width=.23\textwidth]{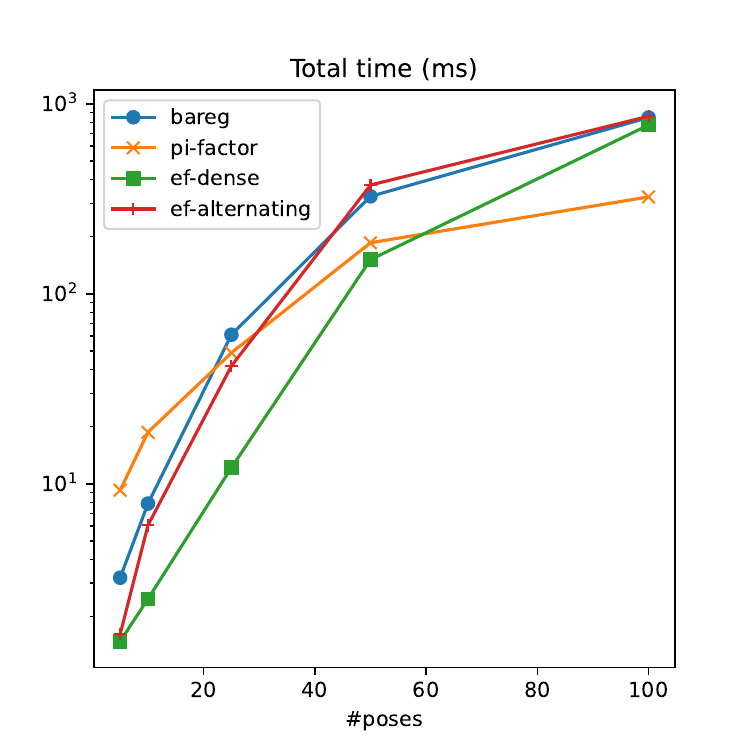}
    \caption{Synthetic environment for 25 planes. {\em Left}: Median time per iteration. {\em Right}: total median time per optimization.}
    \label{fig_synthetic_runtime}
\end{figure}

\subsection{RGBD}
\label{sec_rbgd}

\begin{figure}[h]
    \centering
    \includegraphics[width=.46\textwidth]{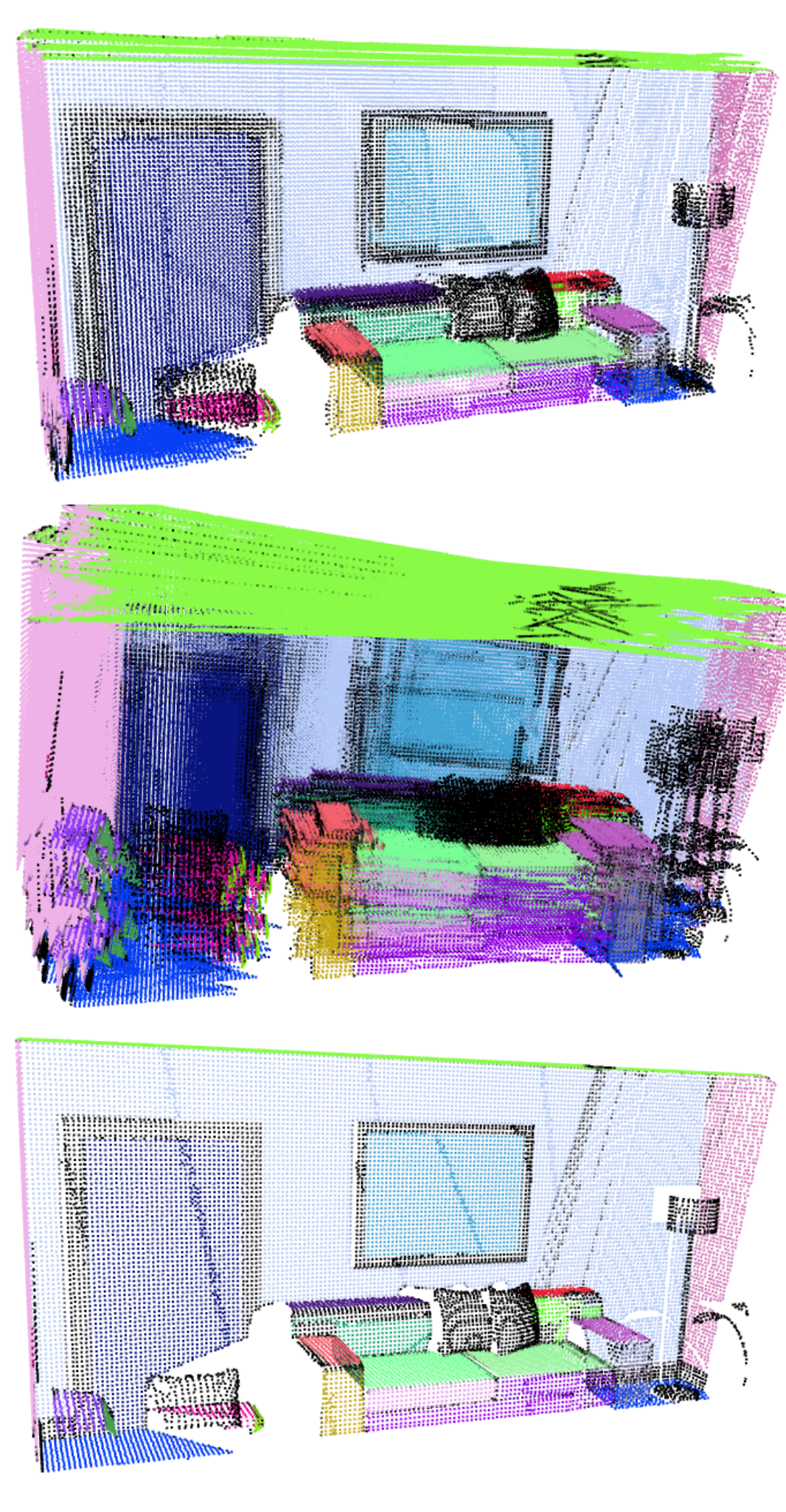}
    \caption{\emph{Top}: Example of the map aggregated from 10 poses with amount of noise 1 degree, 0.01 meter per pose. \emph{Middle}: Example of map aggregated from 10 poses with amount of noise 5 degree, 0.05 meter per pose. \emph{Down}: Result of alignment using the EF algorithm.}
    \label{fig_icl_example}
\end{figure}

The second stage of evaluation includes quality estimation on RGBD data. Since the evaluation considers only back-end processes, we take RGBD sequences from the EVOPS benchmark~\cite{kornilova2022} with pre-labeled planes. This benchmark has labeled data for ICL NUIM and TUM RGBD dataset.
TUM RGBD has only small pieces of trajectories covered by enough amount of planes, therefore only ICL NUIM data is used.

\textbf{Setup.} For evaluation, we randomly sample 50 different subsequences, whose length is 30 consecutive frames. For every subsequence, the poses in original trajectory are perturbed in the range from 0.001-0.15 degrees and 0.001-0.15 meters. Two configuration of ICL NUIM dataset are used: original and one that emulates Kinect noise. For evaluation, Relative Pose Error is used for translation and rotation parts independently.

\textbf{Results and discussion.} A qualitative example of pose perturbation and later alignment by EF is depicted in Fig.~\ref{fig_icl_example}. Since all methods give good map for visual quality assessment, only one method is chosen for visualization.
Results of evaluation with metrics are presented in Fig.~\ref{icl_eval}. On depth data without noise, all methods provide excellent results: $10^{-4}$ is a minuscule error. Surprisingly BALM provides worse results than expected, where in the similar  synthetic the method performed well. The only difference with the synthetic experiment is that planes are observed non-consecutively with poses, so some poses fail to observe a plane, while in the synthetic evaluation, all planes are observed consecutively throughout the trajectory. BAREG performs better since there are more points per plane, favoring their strategy of estimating planes at each pose.

When evaluating the data that emulates Kinect noise~--- EF shows the best performance, followed by EF-Dense up to some perturbation level where the results rapidly degrade. BAREG is performing closely and $\pi$-Factor is following on the performance metrics.
As in the synthetic experiments, EF shows an almost constant error among all ranges of perturbation, indicating good convergence guarantees in practice, while other methods show more variance, even EF-Dense. 
As a consequence, {\bf EF is the preferred method} to optimize under unknown noise conditions and close to real data. This empirical evidence might be supported by the results discussed in sec.~\ref{sec_exp_numerical}, where the process of adding noise degrades the quality of the Hessian estimated and this manifests in worse optimization performance. After all, it is well known how eigenvalues can become unstable for certain conditions, not to mention first or second order derivatives of them.
In addition, this choice is computationally efficient due to the block diagonal structure of the problem, making this result even more relevant.

\begin{figure*}[h!]
    \centering
    \includegraphics[width=.97\textwidth]{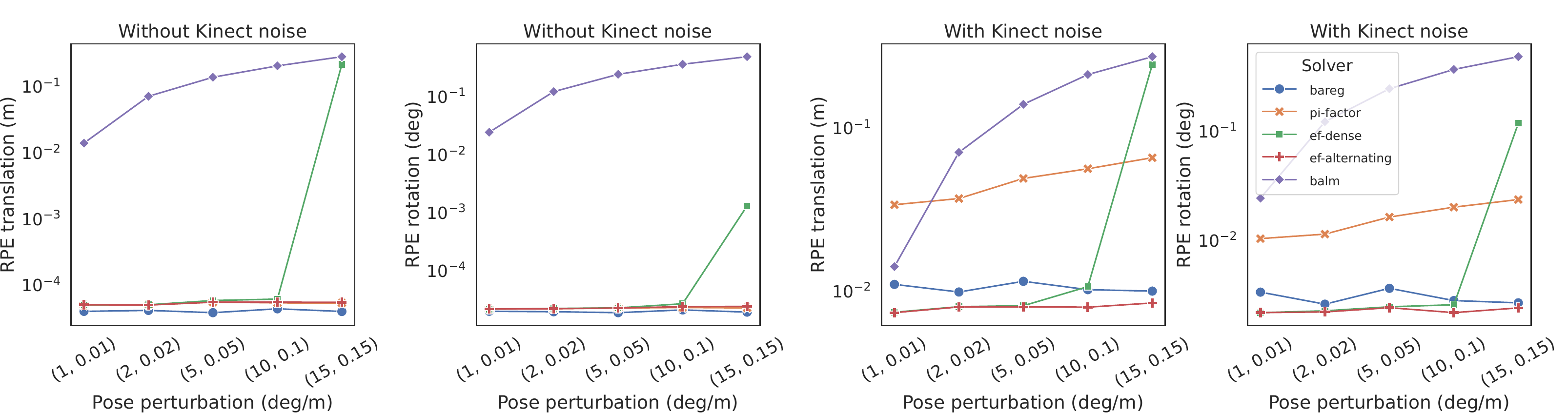}
    \caption{{\em Left:} Evaluation on ICL without noise in the point cloud with respect to perturbations in the initial poses. {\em Right:} Evaluation adding noise to the PC.}
    \label{icl_eval}
\end{figure*}





\subsection{LiDAR}

Finally, methods are evaluated on KITTI dataset with real LiDAR data.

\textbf{Setup.} Original ground truth poses of KITTI dataset are not suitable for evaluation using classical pose-based metrics, since they provide noisy data~\cite{kornilova2022}. Therefore, the following scenario is used: ground truth trajectory is used as the initial conditions for all methods, then the methods optimize over pre-labeled planes from the EVOPS benchmark. For evaluation we consider 100 submaps of length 30 poses, randomly sampled from \texttt{map-00}. To estimate the quality of alignment, no-reference map metrics MME \cite{droeschel2014local}, MPV \cite{razlaw2015evaluation} and MOM \cite{kornilova2021} are used.

\textbf{Results and discussion.} No-reference map-metrics statistics over the methods is presented in Tab.~\ref{no_ref_metrics}, BALM is not reported in the results, since it doesn't provide any convergence on those sequences. Following results, EF shows improvement of poses that gives better map quality on all metrics, qualitative demonstration of this can be found in Fig.~\ref{fig_teaser}. Other methods sometimes show improvement over ground truth poses, but not so drastically as EF. It is worth to note, that KITTI LiDAR data is noisy and EF shows good results even in those conditions, support one more time the empirical evidence of using EF (alternating) as the default method for real and uncontrolled conditions.

\begin{figure}[t!]
    \centering
    \includegraphics[width=.46\textwidth]{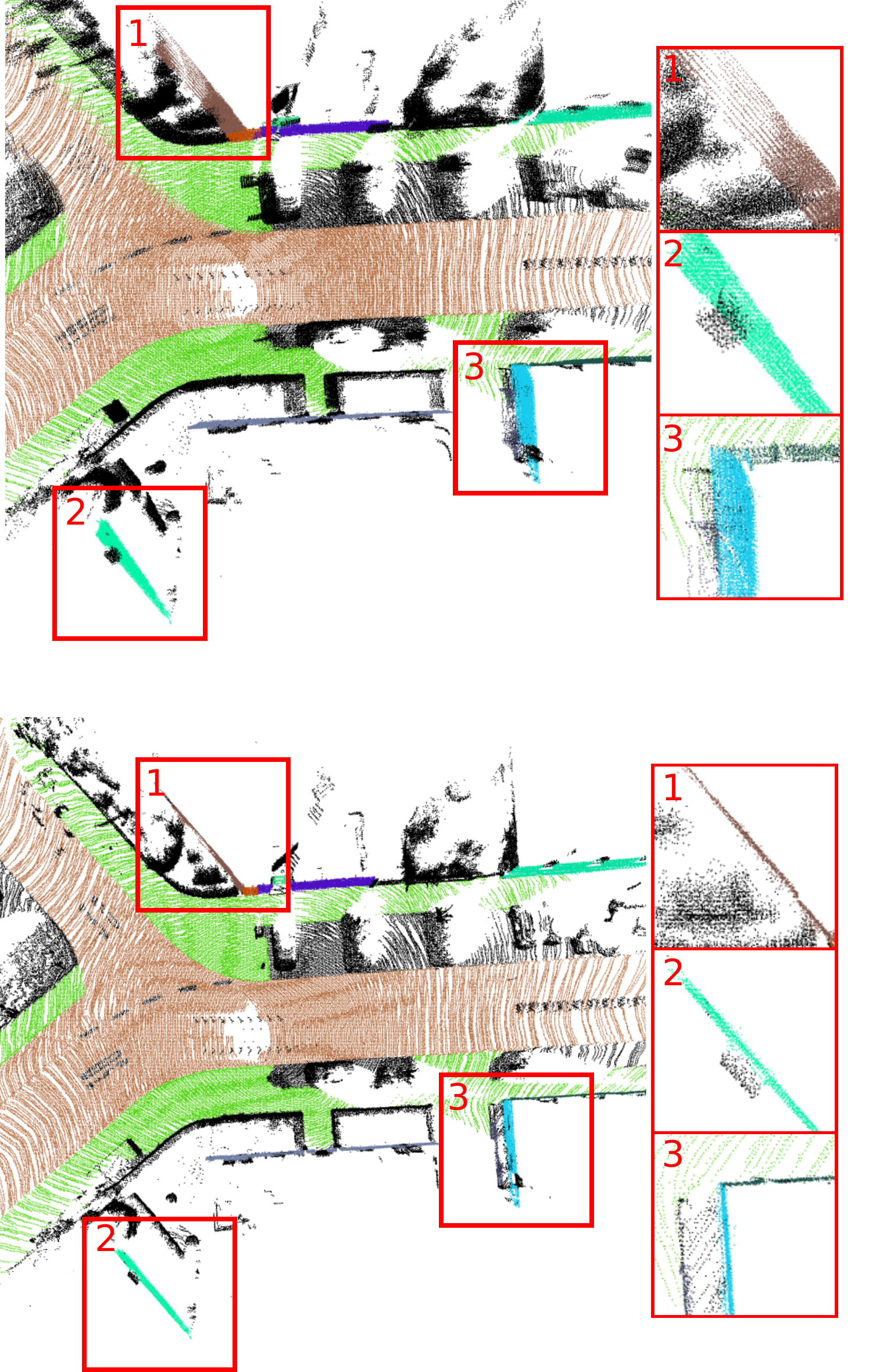}
    \caption{{\em Top:} map aggregated from KITTI ground truth poses. {\em Down:} map aggregated from poses obtained by EF.}
    \label{fig_teaser}
\end{figure}

\begin{table}[h]
\caption{No-reference map metrics on KITTI dataset}
\label{no_ref_metrics}
\begin{tabular}{|c||c|c|c|}
\hline
               & MME  	$\downarrow$ & MPV  	$\downarrow$ & MOM  	$\downarrow$ \\
\hline
BAREG            & 0.24 &  0.0276 & 0.0046 \\
\hline
$\pi$-Factor           & 0.26 & 0.0275 & 0.0028 \\
\hline
GT          & 0.23 & 0.0276  & 0.0029 \\
\hline
EF (ours)          & \textbf{0.17}  & \textbf{0.0266}  & \textbf{0.0023} \\
\hline
\end{tabular}
\end{table}

\section{Conclusions}\label{sec_concl}

This paper is an extension of our previous work on the Eigenfactors \cite{ferrer2019} where we proposed a method to calculate exactly the point error of all points in the plane at $O(1)$ complexity.
The present paper formalizes this concepts and provides an extension, 
from the perspective of manifold optimization as a bilevel program,
and derives an analytical solution for the gradient and the Hessian matrix of the EF by using a simple procedure to calculate derivatives.

We have presented the Eigen Factors method for back-end plane SLAM on two variants, EF-Dense and EF (alternating). 
Approximating the problem to {\em alternating optimization} results in a block diagonal Hessian, and consequently a very efficient solution to the problem. In addition, as an empirical evidence, we found that this approximating form, although requiring more iterations to converge shows excellent properties in synthetic data, and stands apart from other state-of-the-art plane-SLAM solutions by allowing a better convergence to a more accurate solution.

Other methods such as EF-Dense or BALM that calculate an exhaustive and accurate Hessian as a dense matrix, provide good results on low level of perturbation, unfortunately, these methods yield worse results on higher levels of perturbations and more realistic conditions, both on RGBD and LiDAR data.

A practical implication of the EF method is to be used in more realistic environments, due to its robustness compared to other methods in tasks such as map refinement given an initial estimate or allowing loop closure and correct map update, currently not available for most mapping algorithms.

All code is made publicly available at \texttt{https://github.com/prime-slam/EF-plane-SLAM} for the community  as well as python bindings.

\begin{appendices}

\section{EF Hessian derivation}\label{app_hessian}

\begin{proof}

The second derivative is obtained after deriving for each component of $\xi_{t,i}$ and $\xi_{t,j}$, written as $\xi_{i}$ and $\xi_{j}$

\begin{equation}
\frac{\partial }{\partial\xi_j}\left ( \frac{\partial \hat{\lambda}_\pi}{\partial\xi_i}  \right )  = 
\frac{\partial }{\xi_j}\left( \pi^{\top} \frac{\partial \hat{Q}_{t}}{\partial\xi_i}  \pi  \right)
\label{eq_second_derivative_1}
\end{equation}
\begin{align}
\frac{\partial^2 \hat{\lambda}_\pi}{\partial \xi_j \partial \xi_i}  & =  
\frac{\partial \pi^{\top}}{\partial\xi_j} \frac{\partial \hat{Q}_{t}}{\partial\xi_i} \pi +
\pi^\top \frac{\partial^2 \hat{Q}_t}{\partial \xi_j \partial \xi_i} \pi + 
\pi^\top \frac{\partial \hat{Q}_t}{\partial\xi_i} \frac{\partial \pi}{\partial\xi_j} \nonumber \\
& = \pi^\top \frac{\partial^2 \hat{Q}_t}{\partial \xi_j \partial \xi_i} \pi + 2 \frac{\partial \pi^{\top}}{\partial\xi_j} \frac{\partial \hat{Q}_t}{\partial\xi_i} \pi
\label{eq_second_derivative}
\end{align}

Expression (\ref{eq_second_derivative}) is composed of two main terms: the first term depends on the second derivative of the matrix $\hat{Q}_t$
and the second term 
depends on the derivative of the plane, a product of first order derivatives.

The first part:
We calculate the derivative of $\hat{Q}_t$ in (\ref{eq_q_t}) twice, by using the product rule:

\begin{gather}
\frac{\partial ^2 \hat{Q}_{t}}{\partial \xi_{t,j} \partial \xi_{t,i}} = \frac{\partial }{\partial \xi_{t,j}} \left( \frac{\partial\text{Exp}(\xi_t)}{\partial \xi_{t,i}} \hat{Q}_t + \hat{Q}_t \frac{\partial\text{Exp}(\xi_t)^{\top}}{\partial \xi_{t,i}} \right) \nonumber \\
= \frac{\partial ^2\text{Exp}(\xi_t)}{\partial \xi_{t,j} \partial \xi_{t,i}} {Q}_t + \frac{\partial\text{Exp}(\xi_t)}{\partial \xi_{t,i}}  \frac{\partial \hat{Q}_t}{\partial \xi_{t,j}} + \nonumber \\
+  \frac{\partial \hat{Q}_t}{\partial \xi_{t,i}}  \cdot \frac{\partial\text{Exp}(\xi_t)^{\top}}{\partial \xi_{t,j}} + Q_t \frac{\partial ^2\text{Exp}(\xi_t)^{\top}}{\partial \xi_{t,j} \partial \xi_{t,i}}.
\label{eq_hess_1}
\end{gather}

Now considering the multiplication by the plane $\pi$ some symmetries arise:
\begin{gather}
\pi^{\top}\frac{\partial ^2 \hat{Q}_{t}}{\partial \xi_{t,j} \partial \xi_{t,i}} \pi  \nonumber \\ = \pi^{\top} \left( 2\frac{\partial ^2\text{Exp}(\xi_t)}{\partial \xi_{t,j} \partial \xi_{t,i}} \hat{Q}_t +  2 \frac{\partial\text{Exp}(\xi_t)}{\partial \xi_{t,i}}  \frac{\partial \hat{Q}_t}{\partial \xi_{t,j}}\right) \pi \nonumber \\
 = \pi^{\top} 2 \mathcal{H}_{ij} Q_t \pi + 2 \pi^{\top} G_i \frac{\partial \hat{Q}_t}{\partial \xi_{t,j}} \pi^{\top}
\label{eq_hess}
\end{gather}

This form corresponds to (\ref{eq_hessian_diag}) and their corresponding elements in the general Hessian form in (\ref{eq_hessian}).

The condition (\ref{eq_hessian_zero}) can be easily verified by noting that this expression only depends on the matrix $Q_t$ at pose $t$ and the contribution from other poses at $t' \neq t$ vanishes after the first differentiation.

The second part is a modification from \cite{liu2022efficient} adapted to the homogeneous formulation in order to obtain the derivative of the plane $\frac{\partial \pi^{\top}}{\partial\xi_j}$. To this end, we must consider the factorization of the {\em accumulated matrix} $Q$ as

\begin{equation}
Q = U \Lambda U^{\top}  \quad \to \quad \Lambda = U^{\top} Q U,
\label{eq_q_lambda}
\end{equation}
where $U$ is an orthonomal matrix, therefore
\begin{equation}
U^{\top}  U = I \quad \to \quad  \underbrace{\frac{\partial U^\top}{\partial \xi_i} U}_{D^{\top}} +  \underbrace{U^{\top} \frac{\partial U}{\partial \xi_i}}_{D} = 0.
\end{equation}

While in the previous demonstration we considered the plane vector, when considering all the space of solutions by $U$, a new structure arises: the matrix $D$ must be skew-symetric:

\begin{equation}
D = \begin{bmatrix} 0 & d_{12} & d_{13} & d_{14}\\ 
                              -d_{12} & 0 & d_{23} & d_{24}\\
                              -d_{13} & -d_{23} & 0 & d_{34} \\
                              -d_{14} & -d_{24} & -d_{34} & 0
\end{bmatrix}	
\label{eq_matrix_D}
\end{equation}

If we re-formulate the derivation from (\ref{eq_diff_derivation}), now for matrices:
\begin{gather}
\frac{\partial \Lambda}{\partial \xi_i} = \frac{\partial U^\top}{\partial \xi_i} Q U + U^\top \frac{\partial \hat{Q}}{\partial \xi_i} U + U^\top Q \frac{\partial U}{\partial \xi_i} \nonumber \\
= \frac{\partial U^\top}{\partial \xi_i} U \Lambda + U^\top \frac{\partial \hat{Q}}{\partial \xi_i} U + \Lambda U^\top \frac{\partial U}{\partial \xi_i} \nonumber \\
 = D^{\top} \Lambda + U^\top \frac{\partial \hat{Q}}{\partial \xi_i} U + \Lambda D
\label{eq_hess_plane}
\end{gather}
where the equalities $QU=U\Lambda$ and $U^{\top}Q = \Lambda U^{\top}$ from (\ref{eq_q_lambda}) are substituted to simplify the expression in terms of $D$.
One can manipulate (\ref{eq_hess_plane}) for each of its off-diagonal elements (those that are equal to zero) and solve the equation for each element:

\begin{equation}
d_{l,min} = \frac{1}{\lambda_{min}- \lambda_l}u_l^{\top } \frac{\partial \hat{Q}_t}{\partial \xi_{t,i}} u_{min},
\label{eq_d_element}
\end{equation}
where the index of the plane solution $u_{min}$ is already set and $l$ is any other index.

\begin{equation}
D = U^{\top} \frac{\partial U}{\partial \xi_i} \quad \to \quad \frac{\partial U}{\partial \xi_i} = UD
\label{eq_D_solutoin}
\end{equation}

Since $U = [u_1,u_2,u_3,u_4]$, one can operate vectors and to obtain:

\begin{equation}
\frac{\partial U}{\partial \xi_i} = \sum_{l\neq min}^4 \frac{1}{\lambda_{min}- \lambda_l}u_l u_l^{\top} \frac{\partial \hat{Q}_t}{\partial \xi_{t,i}} u_{min}
\label{eq_vector_solutoin}
\end{equation}

Using the expression $\pi = k u_{min}$ (\ref{eq_eigen}), one can substitute in the left and right hand side:

\begin{equation}
\frac{\partial \pi}{\partial \xi_i}k = \underbrace{\sum_{l\neq min}^4 \frac{1}{\lambda_{min}- \lambda_l}u_l u_l^{\top}}_{\mathcal{Q}^{-1}} \frac{\partial \hat{Q}_t}{\partial \xi_{t,i}} \pi k,
\label{eq_plane_partial_sol}
\end{equation}
where the scalar variable $k$ cancels out on both sides.
The result obtained in (\ref{eq_second_derivative}) can then be substituted 
\begin{equation}
2 \cdot \frac{\partial \pi^{\top}}{\partial\xi_{t,j}} \frac{\partial \hat{Q}_{t'}}{\partial\xi_{t',i}} \pi = 
2 \cdot \pi^{\top} \frac{\partial  \hat{Q_t}}{\partial \xi_{t,j}} \mathcal{Q}^{-1} \frac{\partial  \hat{Q}_{t'}}{\partial \xi_{t',i}} \pi
\end{equation}
which exactly corresponds to (\ref{eq_hessian_dense}).

One can note that the expression $\mathcal{Q}^{-1}$ is in fact the inverse of $Q - \lambda_{min}u_{min} u_{min}^{\top}$ noting we are removing one degree of freedom the solution $u_{min}$, hence the notation in (\ref{eq_pseudo_q_inv}) to be an ``inverse''.

\end{proof}

\section{Approximation of the EF Center}
\label{app_center}

Let's first start by defining the centering transformation $T_c$ as 
the retraction for the particular pose at $t$:
\begin{align}
    R_{T_c}(\xi_t) &=  \begin{bmatrix}  I  & -\mu(\xi_t)  \\ 0 &1\end{bmatrix} = \begin{bmatrix}  I  &  -\frac{1}{N} \left( \sum_{\tau \neq t}^H q_{\tau} +  \Exp(\xi_{t})q_t \right)  \\ 0 & 1\end{bmatrix} \nonumber\\
    & = \begin{bmatrix}  I  & \frac{1}{N}(q_t -\Exp(\xi_t) {q}_t)   \\ 0 &1\end{bmatrix} T_c =
    \mathcal{E}(\xi_t) T_c.
\end{align}
where we have used the results from (\ref{eq_q_summation}).
Its derivative at any pose $t$ for the $i$-th coordinate is
\begin{equation}
   \frac{\partial \mathcal{E}(\xi)}{\partial \xi_i} =
   \begin{bmatrix}  0  &  -\frac{1}{N} G_i q_t \\ 0 & 0\end{bmatrix}.
   \label{eq_transl_tretaction}
\end{equation}

Now, we are interested in deriving the terms needed in the optimization, that are affected by the centering transformation. These two equalities are to consider:

\begin{equation}
Q_c = T_c Q T_c^{\top}, \qquad Q = T_c^{-1} Q_c T_c^{-\top}.
\end{equation}

As discussed earlier, only the component of the Hessian depending on the derivative of the plane (\ref{eq_hessian_Q_plane_center}) need to be recalculated using the new centered plane estimation method. Therefore, the expression 

\begin{equation}
\frac{\partial \pi}{\partial\xi_j} = \frac{\partial}{\partial\xi_j} \left( T_c^{\top}\begin{bmatrix} \eta \\ 0 \end{bmatrix}\right)
\end{equation}
can be viewed as a transformation of the plane (\ref{eq_center_plane_vector}). If we develop this expression further, we obtain two terms:

\begin{equation}
\frac{\partial \pi}{\partial\xi_j} = \frac{\partial T_c^{\top}}{\partial\xi_j} \begin{bmatrix} \eta \\ 0 \end{bmatrix} +
T_c^{\top} \frac{\partial}{\partial\xi_j}\begin{bmatrix} \eta \\ 0 \end{bmatrix} \approx T_c^{\top} \frac{\partial}{\partial\xi_j}\begin{bmatrix} \eta \\ 0 \end{bmatrix}
\label{eq_plane_transformed_deriv}
\end{equation}

This approximation is motivated after further expansion:

\begin{equation}
\frac{\partial T_c^{\top}}{\partial\xi_j} \begin{bmatrix} \eta \\ 0 \end{bmatrix} =
T_c^{\top} \frac{\partial \mathcal{E}(\xi_j)^{\top}}{\partial\xi_j} \begin{bmatrix} \eta \\ 0 \end{bmatrix} = 
\begin{bmatrix} 0 \\ -\frac{1}{N} (G_i q_t)^{\top}\eta \end{bmatrix}
\label{eq_partial_eta_center}
\end{equation}
where we observe that the derivative of the normal $\eta$ is zero and only the variable $d$ of the plane is perturbed by a small number inverse of the number of points.

To solve the second expression, we follow a similar approach than in (\ref{eq_hess_plane}). To this end, we note that the $3 \times 3$ matrix to estimate the plane (centered):

\begin{equation}
Q_{p} - \frac{1}{N}q\cdot q^{\top} = V \Lambda_c V^{\top}
\end{equation}
and if looking at the homogenous centered matrix
\begin{equation}
Q_c =  \begin{bmatrix} V\Lambda_c V^{\top} & 0 \\ 0 & N \end{bmatrix}
\end{equation}
we can extend one dimension the eigenvalues such that
\begin{equation}
U = \begin{bmatrix} v_1 & v_2 & v_3 & 0 \\ 0 & 0 & 0 & 1 \end{bmatrix}
\end{equation}
and we repeat the same procedure as in (\ref{eq_vector_solutoin}).
Therefore, the results can be grouped into:

\begin{equation}
\frac{\partial}{\partial\xi_j}\begin{bmatrix} \eta \\ 0 \end{bmatrix} = \underbrace{\begin{bmatrix} \sum_{l\neq min}^3 \frac{1}{\lambda_{min}- \lambda_l}v_l v_l^{\top} & 0 \\ 0 & \frac{1}{(\lambda_{min}-N)} 
    \end{bmatrix}}_{\mathcal{Q}_c^{-1}}
    \frac{\partial \hat{Q}_c}{\partial \xi_{j}} \begin{bmatrix} v_{min} \\ 0 \end{bmatrix}.
\label{eq_eta_deriv_center}
\end{equation}

We have obtained $\mathcal{Q}_c^{-1}$ as in (\ref{eq_hessian_Q_plane_center}).
The last step requires defining the derivative of the centered accumulated matrix $Q_c$ in terms of known variables. For that, we expand it by
\begin{equation}
\frac{\partial \hat{Q}_c}{\partial \xi_{j}}  = \frac{\partial T_c \hat{Q} T_c^{\top}}{\partial \xi_{j}} \approx T_c \frac{\partial \hat{Q}}{\partial \xi_{j}}  T_c^{\top}
\end{equation}
where we have approximated the derivatives of $T_c$ to be negligible, as motivated in (\ref{eq_partial_eta_center}).
Next, we can write the derivative of the plane using (\ref{eq_plane_transformed_deriv}) and (\ref{eq_eta_deriv_center})

\begin{equation}
\frac{\partial \pi}{\partial\xi_j} = T_c^{\top} \mathcal{Q}_c^{-1} T_c \frac{\partial \hat{Q}}{\partial\xi_j} \underbrace{T_c^{\top} \begin{bmatrix} \eta \\ 0 \end{bmatrix}}_{\pi}.
\end{equation}

One can write all components in the same expression

\begin{equation}
\frac{\partial \pi^{\top}}{\partial\xi_j} \cdot \frac{\partial  \hat{Q}_{t'}}{\partial \xi_{t',i}} \pi =  \pi^{\top} \frac{\partial  \hat{Q_t}}{\partial \xi_{t,j}} T_c^{\top} \mathcal{Q}_c^{-1} T_c \frac{\partial  \hat{Q}_{t'}}{\partial \xi_{t',i}} \pi
\end{equation}
which coincides with the expression in (\ref{eq_hessian_plane_center}) as intended.

\end{appendices}


\end{document}